%% file: arxiv_bmvc.tex
\documentclass{bmvc2k}

\input{preamble_n_arxiv}

\input{defs}
\title{Projected Stochastic Gradient Descent with Quantum Annealed Binary Gradients}

\addauthor{Maximilian Krahn}{}{1,2}
\vspace{-10cm}
\addauthor{Michele Sasdelli}{}{3}
\addauthor{Fengyi Yang}{}{3}
\addauthor{Vladislav Golyanik}{}{4}
\addauthor{Juho Kannala}{}{1}
\addauthor{Tat-Jun Chin}{}{3}
\addauthor{Tolga Birdal}{}{2}

\addinstitution{
 Aalto University, \\
 Espoo, Finland
}
\addinstitution{
Imperial College London, \\
London, UK
}
\addinstitution{
University of Adelaide, \\
Adelaide, Australia
}
\addinstitution{
MPI for Informatics, \\
Saarbrücken, Germany
}

\runninghead{Krahn et al.}{QP-SBGD}

\def\etal{\emph{et al}\bmvaOneDot}

\begin{document}

\maketitle

\begin{abstract}
We present {\name} (\nameShort), a novel per-layer stochastic optimiser tailored towards training neural networks with binary weights, known as \emph{binary neural networks} (BNNs), on quantum hardware. BNNs reduce the computational requirements and energy consumption of deep learning models %
with minimal loss in accuracy. However, training them in practice remains to be an open challenge. Most known BNN-optimisers either rely on projected updates or binarise weights post-training. Instead, \nameShort~approximately maps the gradient onto binary variables, by solving a \emph{quadratic constrained binary optimisation}. Moreover, we show how the $\mathcal{NP}$-hard projection can be effectively executed on an adiabatic quantum annealer.
We prove that if a fixed point exists in the binary variable space, the updates will converge to it. 
Our algorithm is implemented per layer, making it suitable for training larger networks on resource-limited quantum hardware. Through extensive evaluations, we show that \nameShort~outperforms or is on par with competitive and well-established baselines such as BinaryConnect, signSGD and ProxQuant when optimising binary neural networks.
\end{abstract}

\maketitle

\input{content/introduction}
\input{content/related_work}

\input{content/method}

\input{content/experiments_nips}

\input{content/conclusion}

{\small
\bibliography{biliography}
}

\section*{Appendix}

\input{content_appendix/appendix_nips_arxiv}

\input{content_appendix/misc}
\input{content_appendix/experiments}
\input{content_appendix/related_work_extended}

\end{document}

%% file: preamble_n_arxiv.tex
\usepackage{hyperref}
\usepackage{adjustbox}
\usepackage{amsthm}
\usepackage{algorithm}
\usepackage{algpseudocode}
\usepackage{placeins}

\usepackage{enumitem}

\usepackage{newtxtext,newtxmath}

\usepackage{wrapfig}
\usepackage[utf8]{inputenc} %
\usepackage[T1]{fontenc}    %
\usepackage{hyperref}       %
\usepackage{url}            %
\usepackage{booktabs}       %
\usepackage{amsmath}
\usepackage[capitalize]{cleveref}
\usepackage{amsfonts}       %
\usepackage{nicefrac}       %
\usepackage{microtype}      %
\usepackage{xcolor}         %
\usepackage{subcaption}
\usepackage{graphicx} 
\usepackage{comment} 
\usepackage{etoolbox}
\usepackage{cancel}
\usepackage{soul}

\title{Projected Stochastic Gradient Descent with Quantum Annealed Binary Gradients}

\newcommand{\nametwo}{Projected Stochastic Binary-Gradient Descent}
\newcommand{\name}{Quantum Projected Stochastic Binary-Gradient Descent}

\newcommand{\nameShortTwo}{P-SBGD}
\newcommand{\nameShort}{QP-SBGD}

%% file: defs.tex
\newcommand{\bm}[1]{\mathbf{#1}}

\theoremstyle{plain}

\theoremstyle{definition}

\newcommand{\bU}{\mathbf{U}}
\newcommand{\bj}{\mathbf{j}}
\newcommand{\bu}{\mathbf{u}}
\newcommand{\br}{\mathbf{r}}
\newcommand{\bb}{\mathbf{b}}

\newcommand{\bs}{\mathbf{s}}
\newcommand{\bx}{\mathbf{x}}
\newcommand{\x}{\mathbf{x}}

\newcommand{\bR}{\mathbf{R}}

\newcommand{\y}{\mathbf{y}}
\newcommand{\X}{\mathbf{X}}
\newcommand{\bPi}{\mathbf{\Pi}}

\newcommand{\J}{\mathbf{J}}

\newcommand{\by}{\mathbf{y}}

\newcommand{\z}{\mathbf{z}}
\newcommand{\bz}{\mathbf{z}}

\newcommand{\bZ}{\mathbf{Z}}

\newcommand{\bW}{\mathbf{W}}

\newcommand{\bQ}{\mathbf{Q}}

\newcommand{\cB}{\mathcal{B}} 

\newcommand{\cD}{\mathcal{D}}

\newcommand{\bv}{\mathbf{v}}

\newcommand{\Z}{\mathbf{Z}}

\newcommand{\W}{\mathbf{W}}

\newcommand{\bg}{\mathbf{g}}

\newcommand{\R}{\mathbb{R}}

\newcommand{\bOmega}{\bm{\Omega}}

\DeclareMathOperator{\sign}{sign}

\newtheorem{thm}{Theorem}
\newtheorem{remark}{Remark}

\newtheorem{prop}{Proposition}
\newtheorem{dfn}{Definition}
\newtheorem{asm}{Assumption}

\crefname{algorithm}{\text{Alg.}}{\text{Alg.}}
\crefname{assumption}{\textbf{H}}{\textbf{H}}
\crefname{asm}{\text{Assumption}}{\text{Assumption}}
\crefname{equation}{\text{Eq}}{\text{Eq}}
\crefname{definition}{\text{Dfn.}}{\text{Dfn.}}
\crefname{lemma}{\text{Lemma}}{\text{Lemma}}
\crefname{dfn}{\text{Dfn.}}{\text{Dfn.}}
\crefname{thm}{\text{Thm.}}{\text{Thm.}}
\crefname{tab}{\text{Tab.}}{\text{Tab.}}
\crefname{fig}{\text{Fig.}}{\text{Fig.}}
\crefname{table}{\text{Tab.}}{\text{Tab.}}
\crefname{figure}{\text{Fig.}}{\text{Fig.}}
\crefname{section}{\text{Sec.}}{\text{Sec.}}
\crefname{prop}{\text{Prop.}}{\text{Prop.}}
\crefname{algorithm}{\text{Alg.}}{\text{Alg.}}

\usepackage{color}
\definecolor{darkorange}{rgb}{1.0, 0.45, 0.0}

\newcommand{\MK}[1]{{\color{black}#1}}

\newcommand{\revised}[1]{{\color{black}#1}}

\newcommand{\etal}{\emph{et al.}}
\newcommand{\ie}{\emph{i.e.}}

\DeclareMathOperator*{\argmin}{arg\min}

\renewcommand{\paragraph}[1]{{\vspace{1mm}\noindent \bf #1}.}

\renewcommand{\paragraph}[1]{{\vspace{1mm}\noindent \bf #1}.}

\renewenvironment{equation*}{\[}{\]\ignorespacesafterend}

%% file: content/introduction.tex
\section{Introduction}

Our contemporary times witness the emergence of two exciting and prominent avenues of scientific advancement: quantum computing (QC)~\cite{divincenzo1995quantum,nielsen2002quantum} in physics and
machine learning~\cite{krizhevsky2017imagenet,zhang2020deep} in computer science. QC seeks to develop novel, efficient computing systems that can solve problems beyond the capabilities of classical computers. 

On the other hand, machine learning aims to create algorithms that can analyze and learn from data without the need for explicit programming. With these advancements in mind, it is natural to pose the question of whether quantum computers can be utilised to train our learning machines or in other words, to \emph{optimise} their parameters.

\input{content/teaser2}

This seemingly simple question calls for a 
reconsideration of our training algorithms, due to an important distinction between classical optimisation and the state-of-the-art quantum annealers (QA), such as the D-Wave Advantage~\cite{dwavehandbook,Boothby2020arXiv}\footnote{Quantum annealers based on the \emph{adiabatic quantum computation} have demonstrated remarkable progress in the scale of quantum computing, reaching beyond 5000 qubits, compared to the \emph{universal gate machines} such as Google Sycamore (53 qubits) and IBM Osprey (433 qubits).}. 
While classical optimisation algorithms commonly use \emph{relax-and-round} schemes to tackle discrete problems, QAs excel in solving
combinatorial, binary optimisation problems that can be expressed in the form of Ising models~\cite{lucas2014ising}. Therefore, any training algorithm looking to leverage QA should be modified to involve (sub-)problems in the form of an Ising model. For standard training algorithms, this is not the case.

In this work, we fill this gap for the case of \emph{binary neural networks} (BNNs)~\cite{qin2020binary,wang2021bi}.
Due to their binary nature, training BNNs effectively involves solving a combinatorial optimisation problem supported by real quantum hardware such as quantum annealers (QAs). Once trained, BNNs can be deployed on edge devices for efficient inference. \MK{Here, they can play a vital role in efficient sensors in AR setups. Further, BNNs can be used for real-time image processing.}
To train BNNs on quantum hardware, we introduce \revised{\name}~(\nameShort), a hybrid classical-quantum, stochastic projected-gradient descent optimiser where the gradients are approximated by their projections onto the set of binary variables by solving a Quadratic Unconstrained Binary Optimisation (QUBO) problem~\cite{lucas2014ising} on a real quantum annealer. In contrast to projection-based optimisers~\cite{qin2020binary,bai2018proxquant},~\nameShort~respects the binary nature of the parameter space and potentially delivers higher quality weight updates.

Our algorithm's most computationally expensive step is the projection onto the non-convex set of binary variables, which adheres to a QUBO form and hence can be computed on a quantum annealer.
\MK{In our experiments, we show an advantage using the QUBO binarisation over the simple sign binarisation, used in signSGD \cite{bernstein2018signsgd}.}
As quantum hardware resources are limited, closed-form, one-shot approaches such as that of Sasdelli~\etal~\cite{sasdelli2021quantum} fail to train 
even modest-sized neural networks. As a remedy, we show how to deploy~\nameShort~scalably, by computing weight updates per layer and incrementally over data batches, inspired by the recent layer-wise optimisers~\cite{you2019large}. This is illustrated in~\cref{fig:intro}.
\MK{Finally we apply our method to training binary logistic regression, binary MLPs and binary graph convolutional neural networks~\cite{kipf2016semi,wang2021bi} and evaluate on both image and graph representation learning.}
In summary, our contributions are as follows:
\begin{itemize}[itemsep=0pt,leftmargin=*,topsep=0pt]
    \item We propose~\nameShort, a novel, stochastic optimiser tailored for training binary neural networks utilising real quantum hardware.
    \item We prove that our algorithm converges to a fixed point in the binary parameter space under the assumption of the existence of such a point.%
    \item We show an equivalence of our binary projection to a specific QUBO problem, allowing us to implement our algorithm on quantum hardware.
    \item We conduct thorough evaluations for all variants of our algorithm on the actual quantum annealer of D-Wave as well as on classical simulated annealers. These experiments reveal that our method (i) can outperform classical benchmarks in both settings and (ii) can train binary neural networks in a scalable manner on quantum annealers.
\end{itemize}
To the best of our knowledge, this is the first time a general, practical binary neural network is trained on an actual adiabatic quantum computer. Our algorithm is intended to harness the current and upcoming advancements in QC technology, effectively circumventing practical hardware limitations such as those imposed by memory constraints. 
Our implementation can be found at \url{https://qpsbgd.github.io}. 

%% file: content/teaser2.tex
\begin{wrapfigure}{r}{0.5\textwidth}
    \includegraphics[width=\linewidth]{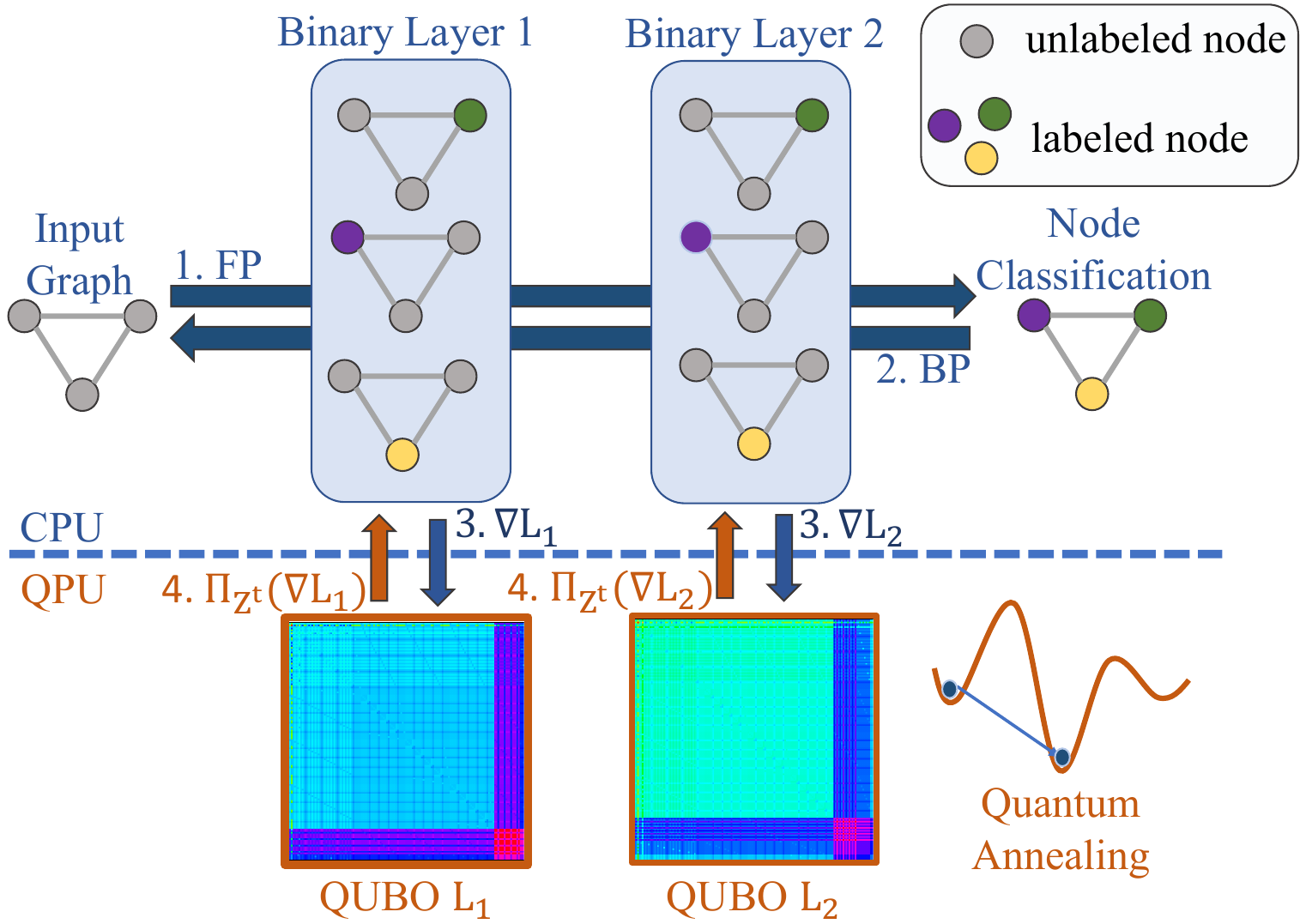}
    \vspace{0.5pt}
    \caption{\footnotesize We propose \textbf{\name} (\textbf{\nameShort}), a provably convergent, layer-wise optimizer for training binary (graph) neural networks on adiabatic quantum annealers. Our hybrid approach iteratively optimizes each hidden layer. We first apply a forward-backward pass on the neural network computing the gradients on a classical computer (steps 1--2).
    We then update the weights in binary variables for each layer separately on the actual quantum hardware of D-Wave~\cite{Boothby2020arXiv}. 
    \vspace{-10pt}
    }
    \label{fig:intro}
\end{wrapfigure}

%% file: content/related_work.tex
\section{Related Work}
\paragraph{Training binary neural networks}
Full-precision neural networks suffer from heavy memory consumption. This has been tackled via various approaches quantising the weights~\cite{gholami2022survey}.
An extreme form of quantisation is restricting the weights to the binary space, yielding BNNs~\cite{BNN_2016}, which
consume less memory and provide faster inference at the cost of certain information loss. Much effort has been devoted to mitigating performance degradation due to binarisation, as detailed in recent surveys~\cite{survey_2020,comprehensive_2021}.
Early, naive methods such as BinaryConnect (BC)~\cite{binaryconnect_2015}, BinaryNet~\cite{BNN_2016}, BinaryNetg~\cite{BNN_2016} or Dorefa-Net~\cite{dorefa-net_2018} aim to binarise weights, activations and during inference and training. Optimisation-based BNNs such as XNOR-Net~\cite{xnor-net_2016}, XNOR-Net++~\cite{bulat2019xnor}, Parameterised Clipping Activation (PACT)~\cite{choi2018pact}, LAB~\cite{hou2016loss}, HWGQ-Net~\cite{cai2017deep} or ReActNet~\cite{liu2020reactnet} attempt to more directly address the accuracy drop resulting from weight and activation binarisation in the naive models, while preserving the compact nature~\cite{survey_2020}. 
Later, ProxQuant \cite{bai2018proxquant} approached binarisation by using a gradual regulariser, which starts with continuous weights and during training gradually regularises those to binary ones.
We refer the reader to Qin~et al.~\cite{survey_2020} for a survey \revised{on standard optimisation techniques.}
Because quantum annealers operate on binary variables, {instead of using an approximation/binarisation scheme, we directly calculate the binary weight updates for the networks.}

\paragraph{Quantum neural networks}
As an emerging field, \emph{quantum machine learning}~\cite{biamonte2017quantum,schuld2015introduction,broughton2020tensorflow,wittek2014quantum} has shown that the training of linear regression, support vector machines and k-means clustering admit QUBO-like formulations~\cite{arthur2021qubo}. The first neural network variants trained on quantum hardware were Boltzmann Machines (BMs) as they naturally lent themselves to quantum annealing~\cite{dixit2021training,biamonte2017quantum,wiebe2014quantum,adachi2015application}. \emph{Quantum deep learning}~\cite{wiebe2014quantum,garg2020advances,kerenidis2019quantum}, \textit{i.e.,} creating quantum circuits that enhance the operations of neural networks by physically realizing them, has emerged to alleviate some of the computational limitations of classical deep learning, thanks to the efficient training algorithms~\cite{kerstin2019efficient,beer2020training}. %
This is different from our approach where we leverage QC as a computational tool to train classical architectures. %
Despite the widespread use of quantum annealers in various related domains such as \emph{quantum computer vision}~\cite{SeelbachBenkner2020,QuantumSync2021,yurtsever2022q}, training neural networks via QAs is less attended. Only recently, Sasdelli~\etal~\cite{sasdelli2021quantum} proposed a \emph{one-shot}, closed-form way to train binary neural networks but this approach is not practically applicable to even small-sized neural networks. We circumvent this problem by adopting a layerwise approach, updating each layer individually, leading to a more scalable algorithm.

%% file: content/method.tex
\section{Method}
Our goal is to design an algorithm for training binary (graph) neural networks that is: 
\begin{enumerate}[leftmargin=*,topsep=0pt,noitemsep]
    \item \textbf{QA-friendly}: utilizing quantum hardware to update the weights while ensuring that the results obtained are accurate and reliable, through a hybrid feedback loop; 
    \item \textbf{Provably convergent}: training should result in the network reaching an 
    optimal state; 
    \item \textbf{Layer-wise compliant}: being able to train individual layers of the network independently, thereby enabling scalable training of large and complex models. 
\end{enumerate} %
To meet these criteria, we introduce a new optimiser, \name~(\nameShort), that maintains the binary nature of weights at all times by projecting the gradients (weight updates) onto the set of binary variables. 

\paragraph{Notation}
We consider a differentiable, Lipschitz continuous function $E_f(\x) := E(f(\x)) : \R^n \rightarrow \R$ defined on $f: \R^n \rightarrow \R^m$, implemented as an MLP or GCN, and the loss function $E:\R^m \rightarrow \R$, where $\y = f(\x)$ denotes the evaluation of $f$ at $\x$ and $n$ is the dimension of the input of $f$. 
Let $K_i$ denote the Lipschitz constant for the $i^{th}$ dimension of $E_f$ and $\bar{K} = \frac{1}{n} \sum_{i = 1}^nK_i$. %
While $t$ denotes the current iteration, $T$ is the number of total iterations. We define $\{\alpha_t\}_t$ as a series of learning rates, $\{ \x^t\}_t$ as the series of iterates, \ie parameters updated over iterations. %
$\tilde{\nabla}_{\x}E_f(\x)$ refers to the \emph{stochastic gradient} w.r.t. $\x$.
We further introduce the \emph{normalized Jacobian} $\bZ^t\in\R^{n\times m}$ whose $i^\mathrm{th}$ column is $\R^n\ni\bz_i^t = \bj^t_i / \| \bj_i^t\|^2 $ where $\bj_i^t = \nabla_{\x} y_i|_{\x^t}$. Note that, $\bj_i^t$ denotes the $i^\mathrm{th}$ column of the Jacobian map $\R^{n\times m}\ni\J^t:= \nabla_{\x} \y$ evaluated at $\x^t$.

\paragraph{Structure} In this section, we will present \nameShort~by introducing an binary optimiser,
which can converge to a local, fixed point\footnote{In our context, a fixed point differs from a local minimum. A given scaled binary gradient may not be sufficient for jumping out of a fixed point.
}.
Next, we will show how to express our binary mapping using a quadratic unconstrained binary optimisation (QUBO). Since our algorithm depends not on the solver but on the solution quality, we will leverage an adiabatic quantum computer (such as D-Wave), resulting in Quantum \nameShortTwo\, (\nameShort). 
In what follows, we will provide \emph{proof sketches} and leave the complete proofs to our supplementary.
We start by presenting our new binary mapping operator before defining \nameShortTwo.
\subsection{Binary map (BM)}
\begin{dfn}[Binary map (BM)] 
Let $\bU=[\bu_i]_i$ and let $\bv \in \R^m$. We define
$\bPi_{\bU}:\R^m \rightarrow \{ \pm 1\}^{n}$ to be the map
\begin{align}\label{eq:BP_paper}
\bPi_{\bU} (\bv) := \argmin_{\bg \in \{-1,1 \}^{n}} \sum_{i=1}^m||v_i - \bg^\top \bu_i||_2^2.
\end{align}
\label{def:bp}\vspace{-3mm}
\end{dfn}
In other words, \cref{eq:BP_paper} looks for a binary vector mapping each column of $\bU$ onto $\bv$ in the least error sense.
We now show that such a map can be used to map gradients onto a binary space.
\begin{prop}
\label{prop:eucledian_to_binary}
    Let $\hat{\bPi}_{\bU} :\R^m\to\R^n$ denote the \emph{relaxed} or continuous version of our projection map. Replacing $\bU$ by $\Z^t$ a normalized gradient w.r.t. $\y$ and using $\tilde{\nabla}_{\y}E_f(\x)$ as an input, $\hat{\bPi}_{\bU}$ satisfies:
    \begin{align}
\label{eq:binary_relax}
    \hat{\bPi}_{\Z^t}(\tilde{\nabla}_{\y}E_f(\x)) = \argmin_{\bb\in\R^{n}} \|\tilde{\nabla}_{\x}E_f(\x) |_{\x^t} -\bb\|_2^2.
\end{align}
\end{prop}
Intuitively, $\hat{\bPi}_{\Z^t}$ acts as a Jacobian, transforming the tangent plane of the gradient {with respect} to $\y$ onto the tangent plane of the gradient {with respect} to $\x$. 
However, our original operator $\bPi_{\Z^t}$ projects onto the binary numbers and not the reals. This non-convex map only approximates $\hat{\bPi}_{\Z^t}$. Hence:
\begin{align}
    \bPi_{\bZ^t}(\tilde{\nabla}_{\y}E_f |_{\x^t}) \approx \argmin_{\bb\in\{-1,1\}^{n}} \|\tilde{\nabla}_{\x}E_f |_{\x^t} -\bb\|_2^2.\vspace{-2mm} 
\end{align}
Our supplementary material shows that our convergence depends not on the quality of this approximation but rather on the direction of the binary projection and the real gradient.

Note that, usually $m\leq n$, leading to~\cref{eq:BP_paper} being under-determined.
Yet, in practice, as we operate on the confined space of binary variables we get valid solutions, one of which can approximate the gradient w.r.t. the input. %
With this motivation, we use $\bPi_{\Z^t}(\tilde{\nabla}_{\y}E_f(\x))$ as an approximation for the true stochastic gradient with respect to $\x$ and introduce our new iterative optimiser, \MK{\nametwo~(\nameShortTwo). First, we introduce the update rule, and in \cref{sec:training_per_layer}, we summarize the algorithm as a whole.}

\subsection{\nametwo~(\nameShortTwo)}
\begin{dfn}[\nameShortTwo]
\label{def:binary_spgd}
\MK{The iterative optimiser admits the following update rule, where we evaluate the gradients on the variables restricted to $\{ \pm 1\}^n$, $\hat{\x}^t =\sign(\x^t)$:}
\begin{align}\label{eq:BAPGDupdate}
    \x^{t+1}=\x^t-\alpha_t \bPi_{\Z^t}(\widetilde{\nabla}_{\y} E_f\left(\hat{\x}^t\right)).
\end{align}
\end{dfn}
\begin{remark}
    This approach bears similarities to a variety of algorithms in the literature. First,~\cref{eq:BAPGDupdate} resembles ProxQuant~\cite{bai2018proxquant} with the critical distinction that while ProxQuant employs the continuous gradient, our approach utilises the binarised gradient directly. Second, by drawing connections to \emph{manifold optimisation}~\cite{boumal2023introduction}, our projection parallels a \emph{Euclidean-to-Riemannian gradient} operator whereas the projected update mirrors an approximate \emph{retraction}. Yet, this is merely an analogy as we operate in a discrete space.
\end{remark}

We now show that if a fixed point exists in the binary set of variables, \nameShortTwo~algorithm can converge to such a point.
\begin{asm}[Divergent learning rates]
The learning rates $\alpha_t$ do not converge, hence $\sum_t \alpha_t = \infty$. Otherwise, we could directly assume the convergence of the algorithm.
\label{ass:LR}
\end{asm}

\begin{thm}[Fixed point of \nameShortTwo]
   Under \cref{ass:LR}, 
   $\mathbf{s} \in \{\pm 1\}^n$ is a fixed point
   for \cref{def:binary_spgd} if and only if $\sign (\bPi_{\mathbf{s}} (\nabla_{\mathbf{y}}E_f(\mathbf{s}))_i) = - \mathbf{s}_i$. This point might not exist, in which case \nameShortTwo~does not converge.
\end{thm}
\begin{proof}
The proof follows analogously to the proof of Prop.~5.3 in Bai \textit{et al.}~\cite{bai2018proxquant}. 
\end{proof}

\vspace{-1mm}
\subsection{\name~(\nameShort)}\vspace{-1mm}
Finding an ideal binary solution for the optimisation problem stated in \cref{def:bp} is an $\mathcal{NP}$-hard problem.
Therefore, we opt for a Quantum Annealer to effectively solve this problem.
To deploy the binary-projection problem onto a QA, we %
{first construct} the corresponding, equivalent Quadratic Unconstrained Binary Optimisation (QUBO) problem:

\begin{prop}[BM as QUBO]\label{prop:ising_for_bp_paper}
The binary map $\bPi_{\bU}(\bv)$ in \cref{def:bp} admits the following \emph{Ising Model} or quadratic unconstrained binary optimisation (QUBO) form:
\begin{align}
    \bPi_{\bU} (\bv) = \argmin_{\bg \in \{-1,1 \}^{n}} \bg^\top\sum_{i=1}^m\bQ_i\bg + \bs^\top\bg, \text{ where }\bs=-2\sum\limits_{i=1}^m v_i\bu_i^t,\quad \bQ_i={\bu_i^t}{\bu_i^t}^\top \quad\mathrm{and}\quad \bQ=\sum\limits_{i=1}^m\bQ_i.\nonumber   
\end{align}

\end{prop}
\begin{proof}[Sketch of Proof]
The derivation follows from the expansion of the Euclidean norm.
\end{proof}
\cref{prop:ising_for_bp_paper} allows us to compute the costly projection step on a quantum annealer such as D-Wave~\cite{dwavepegasus,dwavehandbook}.
To compute the projection for a batch of data, we are using the multi-objective optimisation of~\cite{ayodele2022multi}. With that, we define the quantum versions of our optimisers as follows:
\begin{dfn}[\nameShort]
Whenever we compute the projection by solving the QUBO given in~\cref{prop:ising_for_bp_paper}, we will use the prefix \emph{Q-}, to denote the \emph{quantum}-implementable variant of \nameShortTwo.
\end{dfn}

\subsection{{Training Per Layer}}
\label{sec:training_per_layer}
\begin{wrapfigure}[15]{L}{0.6\textwidth}
\vspace{-15pt}
 \begin{minipage}{0.6   \textwidth}
\begin{algorithm}[H]
\footnotesize
\begin{algorithmic}[1]
\Require Training data $\cD = \{ (\bx_i, \hat{y}_i) \}^{D}_{i=1}$, batch size $B$, learning rate $\alpha$, real initial weights $\{ {\mathbf{\Omega}}^{\ell} \}_{\ell = 0}^{L-1} $
\For{ $t \in [1, \dots, T]$}
 \State $\{ \bW^{\ell} \}_{\ell = 1}^{L} \leftarrow \sign(\bOmega^\ell)$ 
\State Sample a batch index set $\cB \subset \{1, \dots, D\}$.
\State $ \by_{\cB} \leftarrow$ Feedforward pass of $\x_\cB$ %
\State $\{ \dot{\br}^{\ell}_{i,\cB} \}_{\ell = 1}^L \leftarrow$ Compute
intermediate gradients with \cref{eq:pseudo_label}
\For{$\ell = 1,\dots,L$}
\State ${\dot{\bW}}^\ell \leftarrow [\bPi_{\Z^{t,\ell}_{\mathcal{B}, i}}(\dot{\br}^\ell_{i,\mathcal{B}})]_{i = 1}^m$ by solving~the QUBO  \Statex\hskip4.5em defined in~\cref{prop:ising_for_bp_paper}. %
\State ${\mathbf{\Omega}}^{\ell} \leftarrow {\mathbf{\Omega}}^{\ell} - \alpha {\dot{\bW}}^\ell$
\EndFor
\EndFor
\end{algorithmic}
\caption{\footnotesize \name~(\nameShort) }
\label{alg:main_paper}
\end{algorithm}
 \end{minipage}
  \end{wrapfigure}
We now apply \nameShort~layer-wise to train binary \MK{MLPs} -- summarised in~\cref{alg:main_paper}.
We start by considering a general binary neural network where $\R^{m\times n}\ni\X:=\{\x_i\}_{i\leq n}$ denotes the input features and $\X^\ell$ 
is 
the input of the $\ell^{\text{th}}$ layer.
We denote the real-valued weight matrices as
$\{ \mathbf{\Omega}^\ell\}_{\ell=1}^L$,  where $\mathbf{\Omega}^\ell \in \R^{n_{\ell}\times m_{\ell}}$. Similarly, $\W^{\ell} := \sign \mathbf{\Omega}^{\ell}$ denote layer-wise binary weights. 
We first show how to update the weights of each layer, individually. 
We further update each column $j$ of the weight matrix individually.

\begin{dfn}[Intermediate gradient]
\label{def:dotr}
        Let $\cB \subset [1\dots D]$ denote the index for a batch of input vectors. We denote the output vectors of the $\ell^{\mathrm{th}}$ layer corresponding to the batch $\cB$ as $\br_\cB^{\ell}$ and define:
\begin{align}
    \bR^{\ell}_\cB = ((\mathsf{h} \cdots \; \mathsf{h}({\X}\W^{1}) \cdots {\W}^{\ell-1} ){\mathbf{\Omega}}^{\ell})_\cB.
    \label{eq:pseudo_label}
\end{align}
where the binary weights $\W^{\ell} := \sign \mathbf{\Omega}^{\ell}$ pass the information through the non-linear activation function $\mathsf{h}(\X) =\mathsf{hardTanh}(\X)$~\cite{collobert2011natural}.
This is the straight-through estimator (STE), also used in \cite{STE, STE2_2013}.
 The intermediate gradient $\dot{\bR}$ is then defined as:
 $\dot{\bR}^{\ell}_\cB = \frac{\partial E}{\partial \bR^{\ell}_\cB}$, whose $i^{th}$ row is $\br^\ell_{i,\cB}$. 
\end{dfn}

%% file: content/experiments_nips.tex
\section{Experiments}
\label{sec:exp}
In this section, we investigate how well \nameShort~performs on various problem types. 
To do so, we gradually assess the developed optimisers before evaluating \nameShort.
\MK{We start by assessing the convergence of the optimiser} in the binary logistic regression problem.
evaluate the accuracy as well as the convergence of \nameShort~on binary MLPs.
Subsequently, we train and test binary graph-convolutional networks (GCNs) on three different datasets. %
{For our method we are either using the real quantum annealer developed by D-Wave or the well-engineered simulated annealing (SA) software, Gurobi~\cite{gurobi}. Using D-Wave's QA~\cite{dwavehandbook} is marked with (D) and Gurobi's SA with (G)}. 

\paragraph{Baselines}
As this is the \emph{first} quantum deployable optimiser for neural-network training, we compare our algorithm with the binary optimisers deployable on classical hardware. \MK{All chosen baselines also use the principle of layerwise training. Those baselines are, like our proposed method, also general optimisers and not architecture-specific. 
} In particular we consider (ii) \textbf{BinaryConnect} (BC (SGD))~\cite{binaryconnect_2015}, (iii) \textbf{BinaryConnect (signSGD)} (BC (signSGD)), 
a variance of BinaryConnect \cite{binaryconnect_2015}, where the signSGD optimiser is used instead of the proposed SGD to train binary neural networks, and finally (iv) \textbf{ProxQuant}~\cite{bai2018proxquant}.

\subsection{Training binary logistic regression}\vspace{-1mm}

\begin{wrapfigure}[10]{r}{.4\textwidth}
    \centering
    \vspace{-10pt}
    \hspace{-9pt}
    \includegraphics[width=1.05\linewidth]{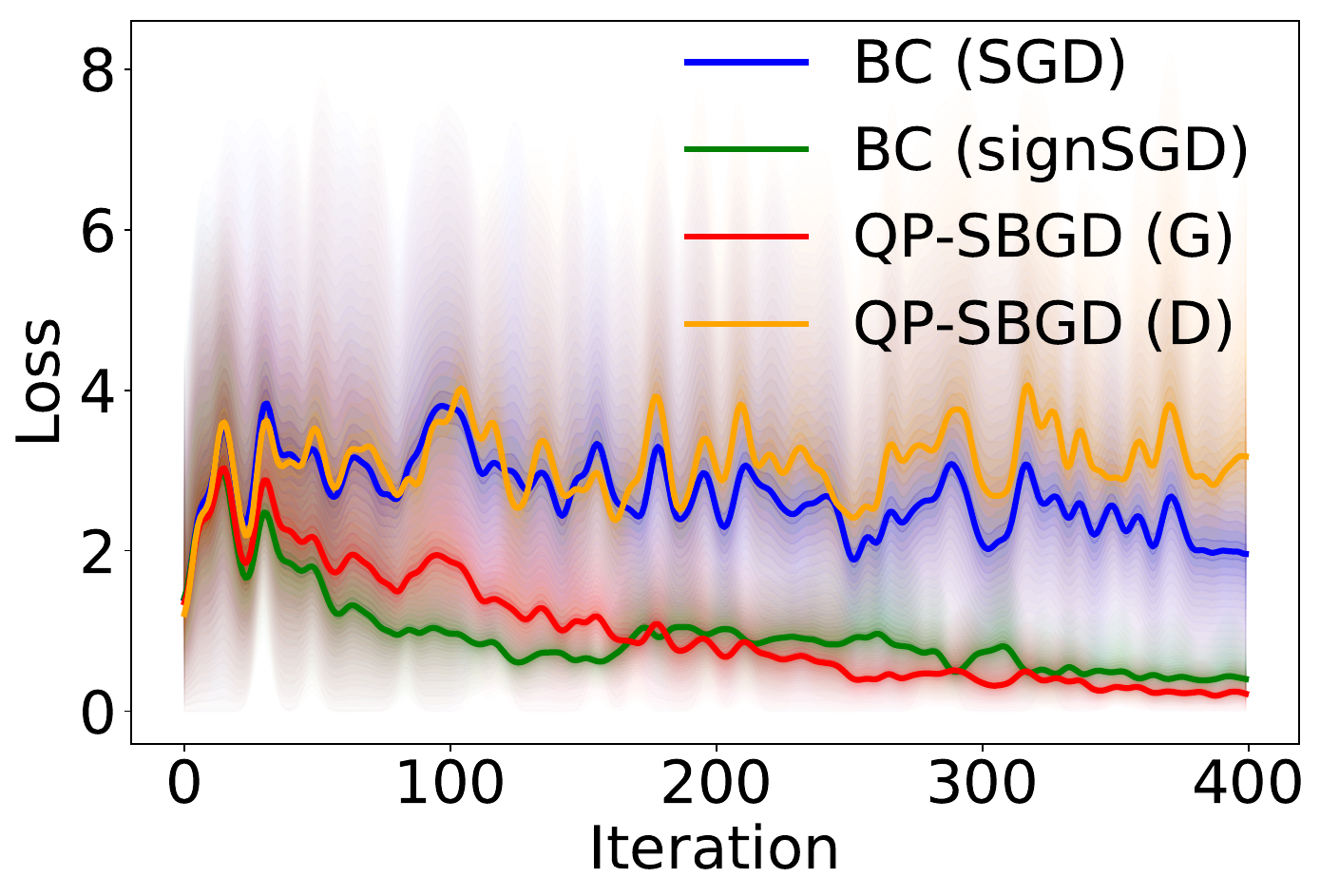}
    \caption{\footnotesize Binary logistic regression.}
    \label{fig:log_regression}
\end{wrapfigure}
Before leaping to neural networks, we evaluate the performance of our algorithm on the well-studied \emph{binary logistic regression}, a logistic regression model where the target variable is binary.
We synthesize a toy dataset to show empirical convergence. First, we sample two, two-dimensional linearly separable blobs, label them with $0$ and $1$, accordingly and lift them into a $3D$ space by appending the $\mathbf{1}$ intercept. %
Next, on this data, we fit the logistic regression model using our algorithm as well as the baselines.
Our model consists of a binary linear layer with an input dimension of $3$ and an output dimension of $1$, followed by a 
\emph{sigmoid} activation function to perform binary classification. It is trained with a binary cross entropy (BCE) loss. %
We use the following learning rates:
$0.05$ for \nameShort~(Gurobi) and BC (signSGD), $0.01$ for \nameShort~(D-Wave) and $5{\cdot}10^{-5}$ for BC (SGD).
In \cref{fig:log_regression}, we observe that after the initial \emph{fitting} phase, our optimizer is consistently the best-performing one when the binary projection is computed via Gurobi. 
This indicates two notable points: (i) our algorithm benefits from solving a QUBO to find the best fitting binary gradient approximation, compared to using $\sign(\cdot)$, and (ii) we are dependent on the quality of projections, \ie, as D-Wave solution has lower quality, our convergence is slower.

\subsection{Training of binary MLPs}

We now present training of binary MLPs for binary classification on two datasets: UCI Adult \cite{platt1998sequential} and MNIST \cite{deng2012mnist}. We experimentally validate the effectiveness of our approach in training binary MLPs by (i) demonstrating that \nameShort~converges and (ii)  \MK{showing accuracies of MLPs trained with different algorithms.} 
\subsubsection{Training on UCI Adult}

The UCI Adult~\cite{platt1998sequential} is a binary classification benchmark containing $123$ binary features. \revised{We use} $1,605$ labelled training data and $30,956$ testing data, \revised{which are selected randomly.} We now analyze the convergence and training accuracy on certain subsets of this dataset. Models in this setup are trained with the BCE loss.

\paragraph{Convergence}
\begin{wrapfigure}[9]{r}{0.5\textwidth}
    \centering
    \vspace{-25pt}
    \includegraphics[width=\linewidth]{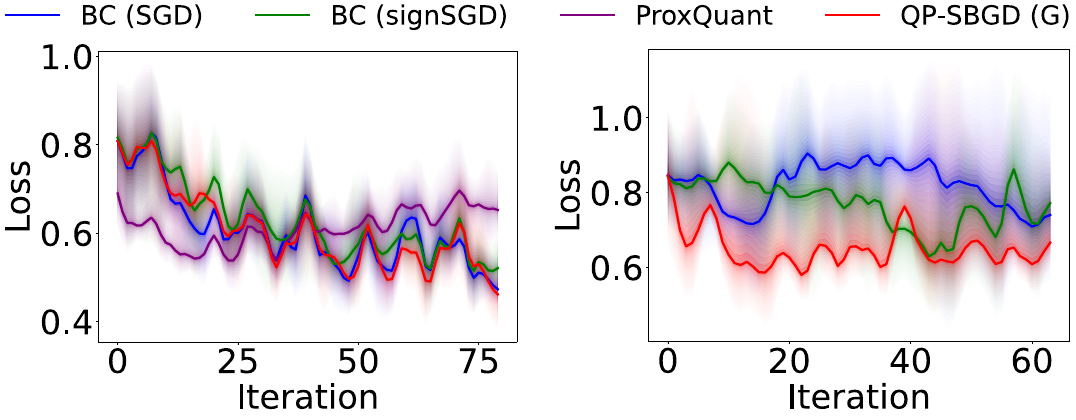}
    \vspace{0.5pt}
    \caption{\footnotesize Mean loss over five runs. \emph{Left:} two- layer setup; \emph{Right:} 10-layer setup. We compare \nameShort~against ProxQuant and BinaryConnect (BC) with signSGD and SGD optimisers. %
    }
    \label{fig:adult_convergance_paper}
  \end{wrapfigure}
We now investigate the convergence of (i) a two-layer BNN with $15$ input dimensions (randomly selected out of $123$ features), $10$ hidden and $1$ output dimension, and (ii) a $10$-layer setup with $123$ input features, a hidden dimension of $128$ and an output dimension of $1$. In both cases, we use 32 batches with a batch size of 16.
~\cref{fig:adult_convergance_paper} demonstrates the convergence of our algorithm in comparison to our baselines. As the training is dominated by noise, we show an average of over five runs here.
The convergence of \nameShort~is similar to BC. Yet, while ProxQuant initially converges faster in the two-layer setup (due to the real weights that gradually convert to binary), it fails to train in the ten-layer setup. In the two-layer setup, BC (SGD) seems to have a slight convergence advantage. However, in the ten-layer network, we observe an advantage of \nameShort~ over the BC versions. While the theoretical convergence rates are identical, ours is the only algorithm with the ability to leverage QCs.

\vspace{-2mm}\subsubsection{Training on MNIST}\vspace{-1mm}
\label{sec:MNIStPaper}

\begin{wraptable}[7]{l}{0.6\textwidth}%
\vspace{-6pt}
\begin{adjustbox}{max width=\linewidth}
\setlength\tabcolsep{2mm}
    \begin{tabular}{lccccc}
    & ProxQuant& BC & BC &\textbf{\nameShort} & \textbf{\nameShort}\\
      &  &  SGD  & signSGD & (Gurobi) & (D-Wave) \\
      \hline
    0/2  & 0.65 &0.64 & \textbf{0.71} & \underline{0.66}  &  0.62\\
    \hline
    1/2 &0.67 & \underline{0.72} & 0.66 & \textbf{0.73} & 0.70\\
    \hline
    1/7  &0.64 &0.74 & 0.68 & \textbf{0.75} & \underline{0.74} \\
    \hline
    \end{tabular}
    \end{adjustbox}
    \vspace{4pt}
  \caption{\footnotesize The accuracy of binary classification on MNIST. The first column contains the digits used in the experiment. \label{tab:MNISTclassification2}}%

\end{wraptable}

We now turn to a famous problem of digit classification on MNIST dataset~\cite{lecun-mnisthandwrittendigit-2010}. 
To ensure that our features are compact and binary, we extract handcrafted features of MNIST letters as follows: 
(i) we extract key points via Monti~\etal~\cite{monti2017geometric}, (ii) we sample 16 lines, which all run through the centroid of the image %
and (iii) for each line, we extract a descriptor by counting the number of key points falling in the positive and negative half-plane. If there are more key points on the positive half-plane of line $i$ we set $s_i=1$ and otherwise to $s_i=-1$. For a given image, this amounts to a feature vector $\mathbf{s}\in\{\pm 1\}^{16}$.
In total, we use a \revised{randomly sampled} train/test split of $500/3000$ images. 
We then adopt a three-layer MLP, with an input dimension of $16$, a hidden dimension of $4$, and an output of $2$ to binary-classify the digits of $0$, $1$, $2$ and $7$ into two classes. {We choose those digits as they are the easiest to distinguish after applying the binarisation.} 
{To train our networks on the resource-limited D-Wave QA, we use a small training sample size of $500$. To have a more representative evaluation we opted for more testing samples.}
The accuracies attained are reported in~\cref{tab:MNISTclassification2}. 
In this experiment, \nameShort~(D) shows to be on par with the state of the art, while our \nameShort~(G) can achieve the best accuracy. With the advancements in quantum hardware, we expect \nameShort~(D) to surpass \nameShort~(G).

\MK{\paragraph{Variance of the runs}
We observed significant variance in individual training runs due to two main factors: 1) the use of binary network weights and 2) the employment of small networks. \revised{Binary weights result in pronounced weight changes during training. Considering also the small network sizes, each weight update significantly influences the output under these circumstances,} amplifying overall variance across training runs. 
}

\begin{figure*}
    \centering
    \includegraphics[width = \textwidth]{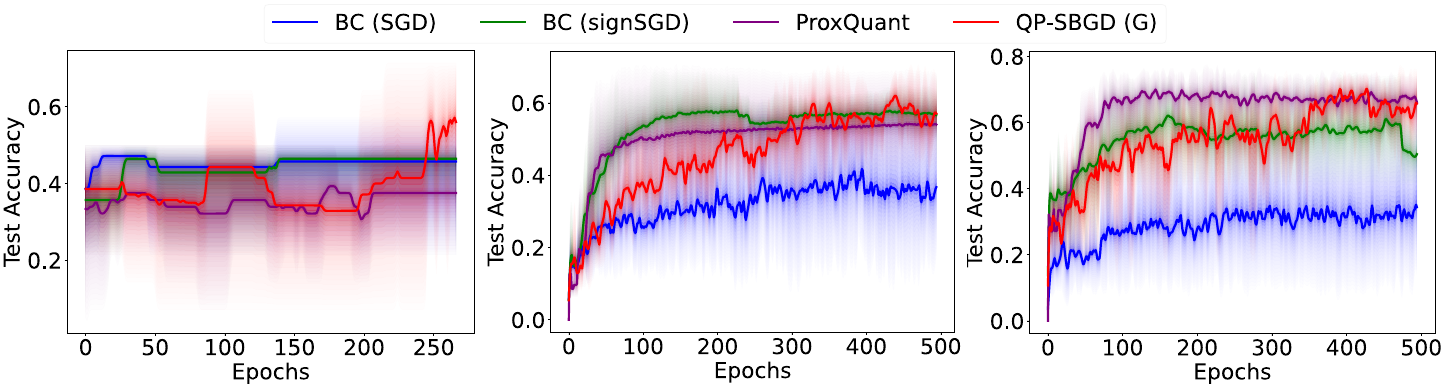}
    \vspace{1pt}
    \caption{\footnotesize \textbf{Graph classification}: We report mean test accuracy over five runs for Karate club~\cite{zachary1977information} (left), Cora~\cite{mccallum2000automating} (middle) and Pubmed~\cite{namata2012query} (right) datasets.\vspace{-4mm}
    }
    \label{fig:acc_on_graphs}
\end{figure*} 

\subsection{Training of binary GCNNs}\vspace{-1mm}

Finally, we present how our algorithm performs on graph data.
We start with tests on the small-scale example of the Karate club social network graph \cite{zachary1977information}. %
In addition, we report the performance of our method on the citation graphs of Cora~\cite{mccallum2000automating} and Pubmed~\cite{namata2012query} benchmarks. %
We start by providing our experimental configuration.
\vspace{-2mm}
\subsubsection{Datasets and Implementation Details}
\vspace{-1mm}
\paragraph{Karate club}
We are using the Karate club social network graph \cite{zachary1977information} with handcrafted features.
The feature vectors one-hot encode the class and thus form an identity matrix.
We reduce those vectors by changing the one hot encoding to a binary encoding, \emph{e.g.}, $0010 \rightarrow 10$.
This reduces the dimensions from $34$ down to $6$.
\MK{As a train/test split we use (for balancing reasons) five nodes from each of the four categories and reserve the remaining 14 for testing.}

\paragraph{Cora and Pubmed}
We use training/validation/test splits in accordance with Yang~\etal
~\cite{yang2016revisiting}: 140/500/1000 for Cora and 60/500/1000 for Pubmed, where the remaining nodes are treated as unlabelled data. 
During training, we use minibatches comprising $1/16^{\mathrm{th}}$ of the training data, from which we draw a single sample per epoch. To reduce the dimensionality of the convolution layers (to $10$ in this case), we use a two-layer MLP. This \emph{encoder}-MLP is in all cases trained with Adam. Nevertheless, the resulting layer-wise QUBO problem, even in the low dimensions, is too large for D-Wave Advantage 6.1. Hence, we use the Gurobi solver to demonstrate the efficacy of~\nameShort. 
On all three setups, we employ the NLL loss.

\vspace{-2mm}
\subsubsection{Results and Discussion} \vspace{-2mm}
Our results, plotted in~\cref{fig:acc_on_graphs}, show that on all three datasets \nameShort~(G) outperforms the baselines on the average.
\revised{Due to the hardware limitations, we employ rather small networks, inherently leading to a higher variance. Hence, we only consider the averages over multiple runs.}
SGD is already known to be a suboptimal optimiser when learning on graphs~\cite{izadi2020optimization}. 
We see that BC also inherits this property of SGD when trained with SGD.
The optimality of our weight update scheme, therefore, improves the effectiveness of training significantly compared to the BC (SGD). 
While BC (signSGD) and ProxQuant exhibit faster convergence compared to \nameShort, our approach achieves superior overall accuracy across all datasets. This once again underscores the inherent advantage of employing a QUBO framework for attaining optimal binary weight updates.
Moreover, the remaining techniques lack quantum deployability, making \nameShort~the pioneering QA-executable algorithm, as far as our knowledge extends. While it is foreseeable that our algorithm encounters limitations imposed either by classical QUBO solvers or quantum hardware, it consistently outperforms classical specialised algorithms such as BinaryConnect and ProxQuant. Further, we expect our algorithm to inherit the advances in quantum computing and naturally improve over time.

%% file: content/conclusion.tex
\section{Conclusion}%

We introduced \nameShort, a scalable hybrid quantum-classical optimiser designed for incremental training of binary neural networks on real quantum hardware. By capitalising on the nature of the optimisation problem associated with incremental binary weight update, we have developed a projection procedure that admits a QUBO form that is effectively solvable via QA. {Under practically acceptable assumptions} we have also provided convergence proofs validating the reliability of our algorithm. Our results demonstrate that \nameShort~achieves on-par or superior performance compared to existing state-of-the-art while remaining adaptable to the advances in quantum computation. 
This further reveals that rethinking our algorithms to be quantum-deployable, brings benefits already for current classical algorithms at hand.

\revised{
\paragraph{Resources} 
The computational cost of our algorithm can be split into two distinct parts: First, the classical SGD-like optimisation and, second, the binarisation of the gradients. For the SGD part, we utilise the backpropagation algorithm that can easily be accelerated with GPUs. It is interesting to see how the QUBO formulation scales with layer width and batch size, i.e.~it scales quadratically in terms of batch size and layer width.}

\paragraph{Limitations}
The problems we can solve on quantum annealers are size-limited by the current capacity of the quantum processing units. 
Nevertheless, we were able to optimise the network weights in 
all our experiments. 
Moreover, due to the rapid improvements in quantum hardware, we are optimistic that our algorithm can impact practical applications in the near future. 
At times, we have found the classical SA solutions to be superior to QAs. We expect this temporary advantage to flip in favour of QAs with improvements in QA hardware.

\paragraph{Future work}
Promising avenues involve (i) QUBOs with larger spectral gaps enhancing the D-Wave solution, (ii) training networks with multiple bits per weight 

\paragraph{Acknowledgements}
This work was supported by the AWS Cloud Credit for Research program. 
The authors acknowledge the support from the Deutsche Forschungsgemeinschaft (DFG, German Research Foundation), project number 534951134. 
Tolga Birdal and Vladislav Golyanik
acknowledge computing resource support from D-Wave.

%% file: content_appendix/appendix_nips_arxiv.tex
{This appendix supplements our paper with the additional material referred to in the main text. This includes the proofs of the propositions (specifically the usage of the binary map as gradient approximation, the convergence and P-SBGD, and the QUBO derivation), supporting evidence of the directionality and additional experimental evaluations, further experiments on the UCI adult dataset, MNIST and analysis of the D-Wave QUBO problems. In addition, we furnish extra information on related work, implementation specifics, and function definitions.}

\section{Theoretical Results}
In \cref{tab:notations} we list the most important notations for the paper and appendix.
\begin{table}[H]
    \centering
    \begin{tabular}{rl}
       Notation  & Explanation\\
       \hline
$f$ & $ \R^n \rightarrow \R^m$any function (neural networks)\\
$n$ & dimension of the input of $f$ \\
$m$ & dimension of the output of $f$ \\
$\y$ & evaluation of $f$ at $\x$; $\y= f(\x)$ \\
 $E(\y)$ & $ \R^n \rightarrow \R$ differentiable, Lipschitz continuous function \\
$ E_f(\x)$ & $\R^n \rightarrow \R$ loss function; $ E_f(\x) := E(f(\x))$ \\
 $K_i$ & Lipschitz constant for the $i^{th}$ dimension of $E_f$ \\ 
       $\bar{K} $ & $ \frac{1}{n} \sum_{i = 1}^nK_i$ \\
       $t$ & current iteration \\
       $T$ & number of total iterations \\
 $\{\alpha_t\}_t$ &  series of learning rates, \\
 $\{ \x^t\}_t$ &   series of iterates, \ie parameters updated over iterations. \\
$\tilde{\nabla}_{\x}E_f(\x)$ &  \emph{stochastic gradient} w.r.t. $\x$ \\
$\bj_i^t$  & $i^\mathrm{th}$ column of the Jacobian map $\R^{n\times m}\ni\J^t:= \nabla_{\x} \y$ evaluated at $\x^t$ \\
$\bZ^t$ & in $ \R^{n\times m}$; \emph{normalized Jacobian} \\
$\bz_i^t $ & in $\R^n$ a column vector fo the normalized Jacobian: $\bj^t_i / \| \bj_i^t\|^2 $ \\

    \end{tabular}
    \vspace{5pt}
    \caption{\footnotesize Notations}
    \label{tab:notations}
\end{table}

\subsection{Approximating Euclidean gradients with binary gradients}
Before we start with the proof we restate \cref{prop:eucledian_to_binary} from the paper.
\begin{prop}%
{
 Let $\hat{\bPi}_{\bU} :\R^m\to\R^n$ denote the \emph{relaxed} or continuous version of our projection map. Replacing $\bU$ by a normalized gradient w.r.t. $\y$ and using $\tilde{\nabla}_{\y}E_f(\x)$ as an input, $\hat{\bPi}_{\bU}$ satisfies:
    \begin{align}
\label{eq:binary_relax}
    \hat{\bPi}_{\Z^t}(\tilde{\nabla}_{\y}E_f(\x)) = \argmin_{\bb\in\R^{n}} \|\tilde{\nabla}_{\x}E_f(\x) |_{\x^t} -\bb\|_2^2.
\end{align}
}

\end{prop}

\begin{proof}
We now give the proof for the continuous map given in~\cref{eq:binary_relax}, where we optimise over $\R^{n}$. We now expand~\cref{eq:binary_relax}:
\begin{align}
    \frac{\partial E}{\partial \x}= \argmin_{\bg \in \R^{n}}\sum_{i = 1}^m|\frac{\partial E}{\partial y_i} - \bg^\top \z_i^t |^2.
    \label{eq:objective_nn}
\end{align}
This statement is true if every term in the sum is $0$. If this holds, we can replace $\bg$ by $\frac{\partial E}{\partial \x}$ and seek to have
\begin{align}
    \frac{\partial E}{\partial y_i} - \frac{\partial E}{\partial \x}^\top\bz_i^t  = 0.
    \label{eq:statement_to_show}
\end{align}
Assuming differentiability and compositionally, we can use the chain rule to write
\begin{equation}\frac{\partial E}{\partial \x} = \frac{\partial E}{\partial y_i} \cdot \frac{\partial y_i}{\partial \x}.
\label{eq:splitGradient}
\end{equation}
By \cref{eq:splitGradient}, and the definition of $\z_i^t$ we can reformulate the left hand side of \cref{eq:statement_to_show} as follows:

\begin{align}
    \frac{\partial E}{\partial y_i} - \left(\frac{\partial E}{\partial \x}\right)^\top  \left(\frac{\partial y_i}{\partial \x} / \|\frac{\partial y_i}{\partial \x} \|^2\right) &=\frac{\partial E}{\partial y_i} - \frac{\partial E}{\partial y_i}\left(\frac{\partial y_i}{\partial \x}\right)^\top \left(\frac{\partial y_i}{\partial \x} / \|\frac{\partial y_i}{\partial \x} \|^2\right) \\
    &= \frac{\partial E}{\partial y_i}\left(1 -  \left(\frac{\partial y_i}{\partial \x} / \|\frac{\partial y_i}{\partial \x} \|\right)^\top \left(\frac{\partial y_i}{\partial \x} / \|\frac{\partial y_i}{\partial \x} \|\right)\right)
   \label{eq:start_of_matrixes}\\
    &=\frac{\partial E}{\partial y_i}(1 - 1) = 0 \qquad \mathrm{for\,all} \,\, i.\label{eq:jacmap}
\end{align}
Alternatively, we can arrive at the same proof by using the fact that Jacobians map tangent vectors, \ie, local isomorphisms between the tangent spaces of input and output points. Finally, by satisfying~\cref{eq:jacmap} for all $i$, we see that in the continuous space, \cref{eq:objective_nn} is fulfilled. 
\end{proof}

Note that we can use this mapping to calculate binary updates also for multidimensional function parameters, such as matrices.
This can easily be done by vectorising the parameter matrix into $m$ vectors and computing $m$ times the mapping of the gradient of a parameter vector.

\subsection{Proof of the fixed point theorem}
Before we prove the fixed point theorem, we present the definition of a fixed point and reiterating the theorem.
\begin{dfn}[Fixed point]
    $\bs \in \{ \pm 1 \}^n$  is a fixed point for \nameShortTwo, 
    if $\bs^0 = \bs$ in \cref{def:binary_spgd} implies that $\bs^t = \bs$ for all $t = 1, 2, ...$, where $\bs^t = \sign(\x^t)$.
    We say that \nameShortTwo~converges if there exists $t < \infty$ such that $\bs^t$ is a fixed point.
\end{dfn}
\begin{thm}[Fixed point of \nameShortTwo]
Let $\Z^s$ be the normalised Jacobian of the function $f$ at $\mathbf{s} \in \{\pm 1\}^n$ . Let $\x^t$ be the iterative produced by Eq. 7 from the paper. We further assume, that for $i \in [n]$ $(\nabla_{\x}E_f(\x^t))_i \neq 0$ and  $(\nabla_{\x}E_f(\bs))_i \neq 0$. Under~the assumption of a diverging learning rate ($\sum_{i \in \mathbb{N}} \alpha_i = \infty$ and $\alpha_i>0$)
   , $\mathbf{s}$ is a fixed point for \nameShortTwo~if and only if ($\Leftrightarrow$) $-\bPi_{\Z^s} (\nabla_{\mathbf{y}}E_f(\mathbf{s})) = \bs$. This point might not exist, in which case we cannot state the convergence of \nameShortTwo~for any starting point $\x^0 \in \R^n$. 
\end{thm}
\begin{proof}
Our proof follows closely the proof of~\cite{bai2018proxquant}. We first define the \emph{fixed point} as a stationary point.
We start with the direction $\Rightarrow$. Hence we assume that $\mathbf{s} \in \{ \pm 1\}^n$ is a fixed point.
We start by restating the update rule for point $\x^T$ starting from a \textbf{fixed point} $\bs$:
\begin{align}
   \x^T = \x^0 - \sum_{t = 0}^{T - 1} \alpha_i \bPi_{\bZ^t}(\nabla_{\y}E_f(\mathbf{s})). %
\end{align}

Hence $\x^T = \x^0$ holds by definition for all $T\in\mathbb{N}_+$. As the fixed point is the same in the space of reals or binary variables, we are allowed to take the sign of both sides:
\begin{align}
   \bs:=\sign(\x^T) = \sign\left(\x^0 - \bPi_{\bZ^{T-1}}(\nabla_{\y}E_f(\bs)) \sum_{i = 0}^{T-1} \alpha_i\right).
\end{align}
Taking the limit $T \rightarrow \infty$ and applying the assumption that $\sum_t \alpha_t=\infty$ yields
\begin{align}
s_i &=\lim _{T \rightarrow \infty} \operatorname{sign}\left(\x^0-\bPi_{\bZ^{T-1}}(\nabla_{\y} E_f(\bs)) \sum_{t=0}^T \alpha_t\right)_i \nonumber \\
&=\sign(-\bPi_{\bZ^{T-1}}(\nabla_{\y} E_f(\bs)))_i)  \\
&=-\bPi_{\bZ^{T-1}}(\nabla_{\y} E_f(\bs)))_i.\nonumber
\end{align}
concluding that $\bs=-\bPi_{\bZ^{T-1}}(\nabla_{\y} E_f(\bs)))$.

Next, we prove the direction " $\Leftarrow$ ". If $\bs$ obeys that $-\bPi_{\bZ^{T-1}}(\nabla_{\y} E_f(\bs))=\bs$ for all $i \in[n]$, then if we take any $\x^0$ such that $\operatorname{sign}\left(\x^0\right)=\bs, \x^t$ will move in a straight line towards the direction of $-\bPi_{\bZ^{T-1}}(\nabla_{\y} E_f(\bs))$. In other words, $\operatorname{sign}\left(\x^t\right)=\operatorname{sign}\left(\x^0\right)=\bs$ for all $t=0,1,2, \ldots$. Therefore, by definition, $\bs$ is a fixed point.
\end{proof}

\subsection{QUBO term of BM}
\begin{prop}[BM as QUBO]\label{prop:ising_for_bp}
The binary projection $\bPi_{\bZ^t}(\bv)$ in \cref{def:bp} admits the following \emph{Ising Model} or quadratic unconstrained binary optimisation (QUBO) form:
\begin{align}\label{eq:BP_as_qubo}
    \bPi_{\bU} (\bv) &= \argmin_{\bg \in \{-1,1 \}^{n}} \bg^\top\sum_{i=1}^m\bQ_i\bg + \bs^\top\bg 
\end{align}
where $\bs=-2\sum_{i=1}^m v_i\bu_i$, $\bQ_i={\bu_i}{\bu_i}^\top$ and $\bQ=\sum_{i=1}^m\bQ_i$.
\end{prop}
\begin{proof}
    
\begin{align}\label{eq:BP}
\bPi_{\bU} (\bv)&= \argmin_{\bg \in \{-1,1 \}^{n}} \sum_{i=1}^m||v_i - \bg^\top \bu_i||_2^2 \\
&= \argmin_{\bg \in \{-1,1 \}^{n}} \sum_{i=1}^m (v_i - {\bu_i}^\top \bg)^\top (v_i - {\bu_i}^\top \bg) \\
&= \argmin_{\bg \in \{-1,1 \}^{n}} \sum_{i=1}^m \bg^\top{\bu_i}{\bu_i}^\top\bg - 2\sum_{i=1}^m v_i {\bu_i}^\top\bg + \sum_{i=1}^m v_i^2 \\
&= \argmin_{\bg \in \{-1,1 \}^{n}} \bg^\top\sum_{i=1}^m\bQ_i\bg + \bs^\top\bg \label{eq:qubo_equation_supp}
\end{align}
where $\bs=-2\sum_{i=1}^m v_i\bu_i$, $\bQ_i={\bu_i}{\bu_i}^\top$ and $\bQ=\sum_{i=1}^m\bQ_i$.

\end{proof}

%% file: content_appendix/misc.tex
\paragraph{Runtime analysis on quantum and central processing units} In general the experiments are relatively easy on the hardware and are run within minutes on a single CPU.
Hence we will only focus on the QPU time, as well as on the graph benchmarks.
For all five Rosenbrock runs we use a total of one minute QPU time.
The small experiments on the Adult dataset as well as the Karateclub network are solved within an hour.
Experiments on graph convolutional neural networks for the bigger datasets Cora and Pubmed run on the CPU for six hours.
In those experiments, we did not need to use any GPU time during training.
The experiments on D-Wave for MLPs are running for 6 hours and one run consumes about a minute of QPU time.
This is explainable with the minor embeddings.
Calculating those minor embeddings is a non-critical combinatorial optimisation problem and, at the moment, is \emph{known} to be a bottleneck for quantum annealers. We are hopeful that those can be accelerated and make all QA algorithms faster, in the future.

To generate all the experiments in our paper we spent an hour of QPU time on a D-Wave quantum annealer as well as 1000h of CPU time.

\section{Implementation Details}

Our code, which we will release, is implemented in PyTorch~\cite{NEURIPS2019_9015}. 
\sloppy To calculate all gradients in all the experiments, we use the torch autograd library. 
In the graph experiments with Cora and Pubmed we use the Adam optimiser \cite{kingma2014adam} with the PyTorch implementation.
For BinaryConnect~\cite{binaryconnect_2015} we use our own implementation. However, contrary to the BinnaryConnect~\cite{binaryconnect_2015} paper we are using our definition of the sign function and not theirs, which projects $\sign(0) = 1$ or does some random assigning.
For ProxQuant we also use our own implementation, following the paper.
The code will be released upon acceptance.

\subsection{MNIST Feature Extraction}
In the paper, we already described how we extract the features of the MNIST images.
In this section, we give a more detailed overview of this procedure as well as a qualitative example in \cref{fig:mnist_generation}.
To this end, we follow a three-step procedure to generate the 16-bit features: (i) We use the key points generated from the method described in \cite{monti2017geometric}. They are marked in \cref{fig:mnist_generation} as blue dots in the middle and right image. Afterwards in step (ii) we sample 16 lines, which all run through the centroid of the image. \cref{fig:mnist_generation} in the right image contains four out of the 16 lines drawn in green.
For the last and (iii) step, we count the number of feature points above and below the line. If there are more key points above line $i$ than below, then we assign the feature $i$ to $1$; otherwise to $-1$. %
\begin{figure}
    \centering
  \includegraphics[width=\linewidth]{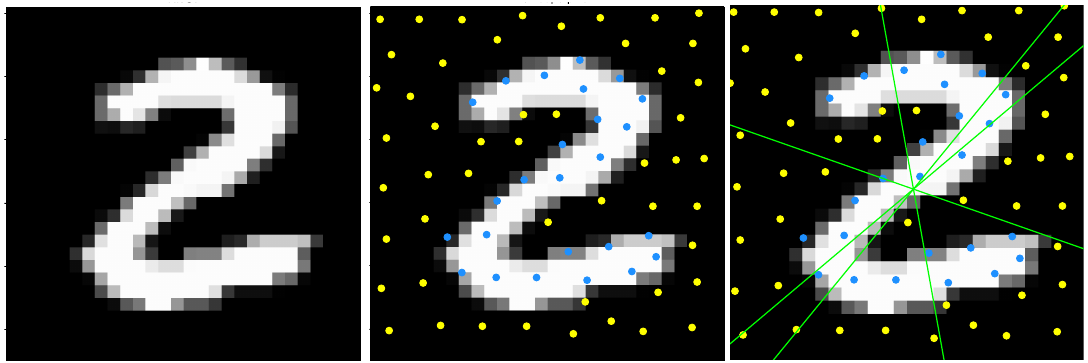}
  \vspace{5pt}
  \caption{\footnotesize Label generation on MNIST \cite{lecun-mnisthandwrittendigit-2010} (example). 
\textit{Left:} Initial image of the digit ``2'';  \textit{Middle:} Extracted keypoint features with the superpixel approach \cite{monti2017geometric}; \textit{Right:} sampling four lines to determine binary features.}%
\label{fig:mnist_generation}
\end{figure}

\subsection{Definitions}
\paragraph{Hard tanh}
Hard tanh is defined as follows:
\begin{align}\label{eq:hardtanh}
\mathsf{hardTanh}(x) = 
\begin{cases}
    -1 & \text{ if } x \leq -1, \\
    x & \text{ if }  -1 < x < 1, \\
    1 & \text{ if } x \geq 1, \\
\end{cases}.
\end{align}

\paragraph{sign}
We use the following definition of the sign function:
\begin{align}\label{eq:sign}
\sign(x) = 
\begin{cases}
    -1 & \text{ if } x < 0, \\
    1 & \text{ if } x \geq 0. \\
\end{cases} 
\end{align}

%% file: content_appendix/experiments.tex
\section{CDP hypothesis test for neural networks}
We will use the $H_0$ hypothesis Z test to determine if our algorithm is more likely to point in the same direction as the gradient.
As a comparison, we will conduct the same test on the established binary baseline signSGD.
We choose as our hypothesis: 

$\mathbf{H_0}$ that the projected gradient points in a random direction, that means:
\begin{align}
\begin{split}
       \rho_i(\x^t) = \operatorname{Prob}\left( \bPi_{\Z^t}(\tilde{\nabla}_\y E_f(\x^t))_i=\operatorname{sign} (\nabla_\x E_f(\x^t))_i\right) = 0.5
   \end{split}
\end{align}

$\mathbf{H_1}$ that the projected gradient points in the same/opposite direction
\begin{align}
\begin{split}
       \rho_i(\x^t)= \operatorname{Prob}\left( \bPi_{\Z^t}(\tilde{\nabla}_\y E_f(\x^t))_i=\operatorname{sign} (\nabla_\x E_f(\x^t))_i\right) > 0.5
       \label{eq:h1hypothesis}
       \end{split}
\end{align}
or
\begin{align}
\begin{split}
       \rho_i(\x^t) 
       = \operatorname{Prob}\left( \bPi_{\Z^t}(\tilde{\nabla}_\y E_f(\x^t))_i=-\operatorname{sign} (\nabla_\x E_f(\x^t))_i\right) > 0.5
       \end{split}
\end{align}

We can clearly see that either $\mathbf{H_0}$ or $\mathbf{H_1}$ is true.

The data collection is done by measuring the gradients during training on the MNIST experiment in the paper, where we use the gradient on the full training set as the real non-stochastic gradient.

Let $k$ be the number of occurrences where the gradient points in the same direction as our projected gradient, $p = 0.5$ as our $H_0$ probability and $n$ as the number of compared gradients. 
The Z test for binary variables is formulated as follows:
\begin{align}
    Z=\frac{k-n p}{\sqrt{n p(1-p)}}.
    \label{eq:zvalue}
\end{align}

A Z value in the interval of $[-1.96, 1.96]$ provides a $95\%$ certainty that $\mathbf{H_0}$ is fulfilled.

The following two tables \cref{tab:zvalues_qspgd} and \cref{tab:zvaluessignsgd} show a sample collection from 5 runs. In the first column, we describe the number of samples, where our projected gradient and the real gradient point in the same direction, in the second column we have all samples and in the last column we note the Z value for the experimental run, which was calculated with \cref{eq:zvalue}.

\paragraph{Discussion}
As we can see in \cref{tab:zvalues_qspgd} and \cref{tab:zvaluessignsgd}, when calculating the gradient with \nameShort~the Z-value is always bigger than $1.96$. This shows that we reject the $\mathbf{H_0}$ hypothesis and accept the $\mathbf{H_1}$ hypothesis.
Hence we can be confident that the CDP holds in our case.
As the tables further present, the projected gradients point in the same direction as the real gradient, then away. Hence we can narrow down the $\mathbf{H_1}$ hypothesis to the CDP assumption.

Further to verify our experiments we also tested signSGD, where the CDP assumption is indicated by a few lemmas. Here we observe that $\mathbf{H_0}$ can be rejected, but on a weaker basis, as the Z values are smaller. This further indicates that our algorithm fulfils the CDP assumption.

\begin{table}
  \begin{minipage}{.5\linewidth}
\centering
    \resizebox{\textwidth}{!}{
    \begin{tabular}{c|c | c}
       $\#$ of same direction & total $\#$ of samples & Z-value\\
        \hline
        2285 & 3915 & 10.4\\
        2199 & 3647 & 12.4 \\ 
        2311 & 3886 & 11.8 \\ 
        2334 & 3984 & 10.8\\
        2356 & 3990 & 11.4
    \end{tabular}
    }
    \vspace{5pt}
    \caption{\footnotesize Samples from \nameShort~directions.}\label{tab:zvalues_qspgd}
    
  \end{minipage}%
  \begin{minipage}{.5\linewidth}
    \centering
    \resizebox{\textwidth}{!}{
     \begin{tabular}{c|c | c}
         $\#$ of same direction & total $\#$ of samples & Z-value\\
        \hline
        2334 & 4626 & 0.61\\
        732 & 1176 & 8.39 \\ 
        339 & 630 & 1.91 \\ 
        2559 & 4702 & 6.06\\
        2103 & 4683 & -6.9
    \end{tabular}
    }
        \vspace{5pt}
   \caption{\footnotesize Samples from signSGD.}\label{tab:zvaluessignsgd}
  \end{minipage}
\end{table}

\section{Additional experiments}

\begin{wrapfigure}[7]{r}{0.3\linewidth}
\vspace{-25pt}
    \centering
    \includegraphics[width=1\linewidth]{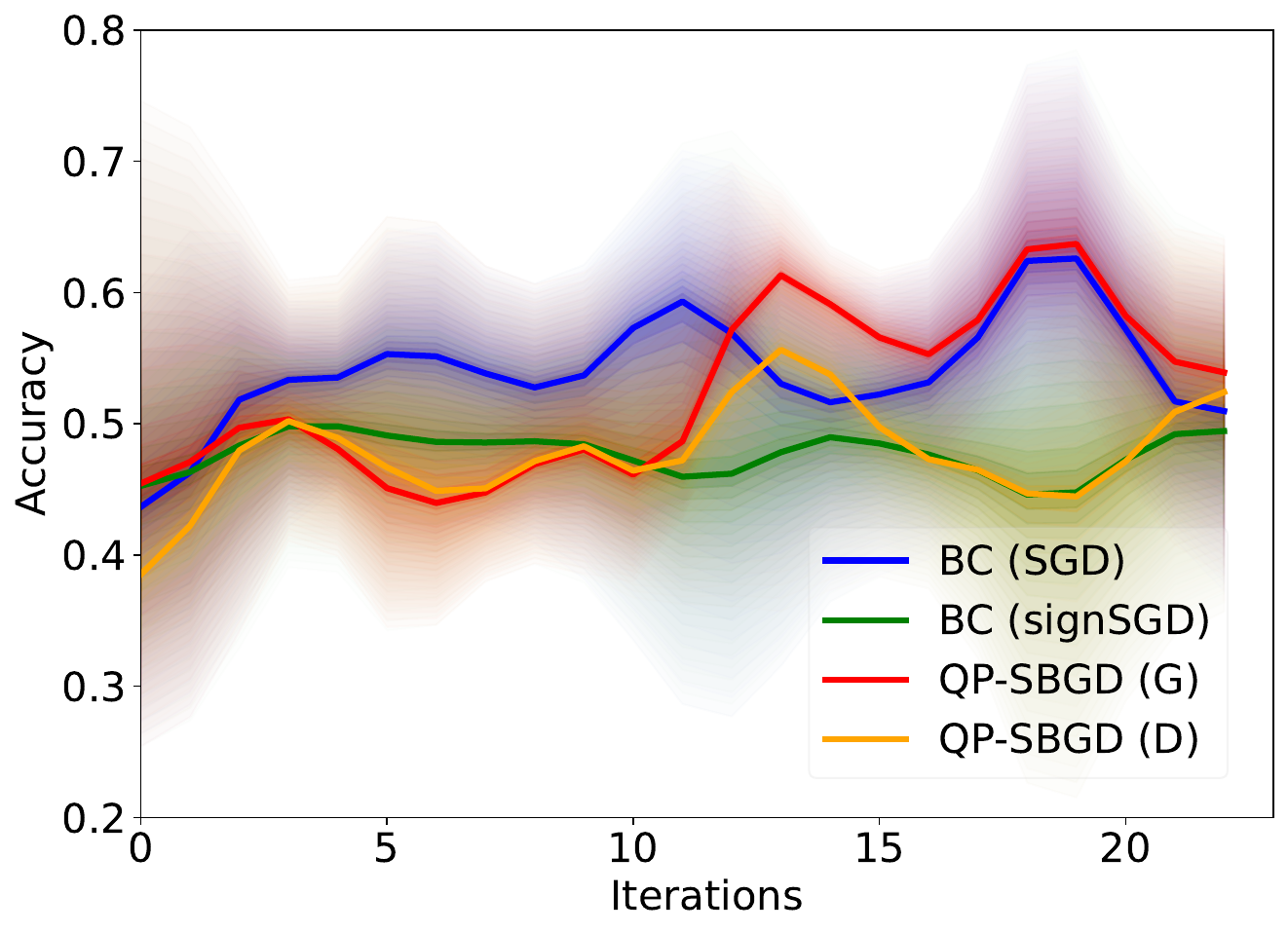}
    \vspace{5pt}
    \caption{\footnotesize Training accuracy on a subset of the UCI Adult dataset for binary classification.}
    \label{fig:adult_training_acc}
\end{wrapfigure}

In this section, we present additional experiments on the UCI adult dataset, some qualitative examples on the MNIST dataset, and some D-Wave analysis.

\subsection{Quantitative results on UCI Adult}
We now use three manually selected features and 30 batches with a batch size of 20 to train a two-layer neural network with the input dimension of $3$, hidden dimension of $5$ and the output of $2$.
\cref{fig:adult_training_acc} reports the training accuracy where we observe that \nameShort~(G) outperforms the baselines as well as our D-Wave optimiser, \nameShort~(D), which nevertheless performs better than or on par with BC (signSGD). 

\subsection{Qualitative Results of MNIST classification}

\begin{figure}
  \begin{minipage}{.45\linewidth}
       \centering
        \includegraphics[width=\linewidth]{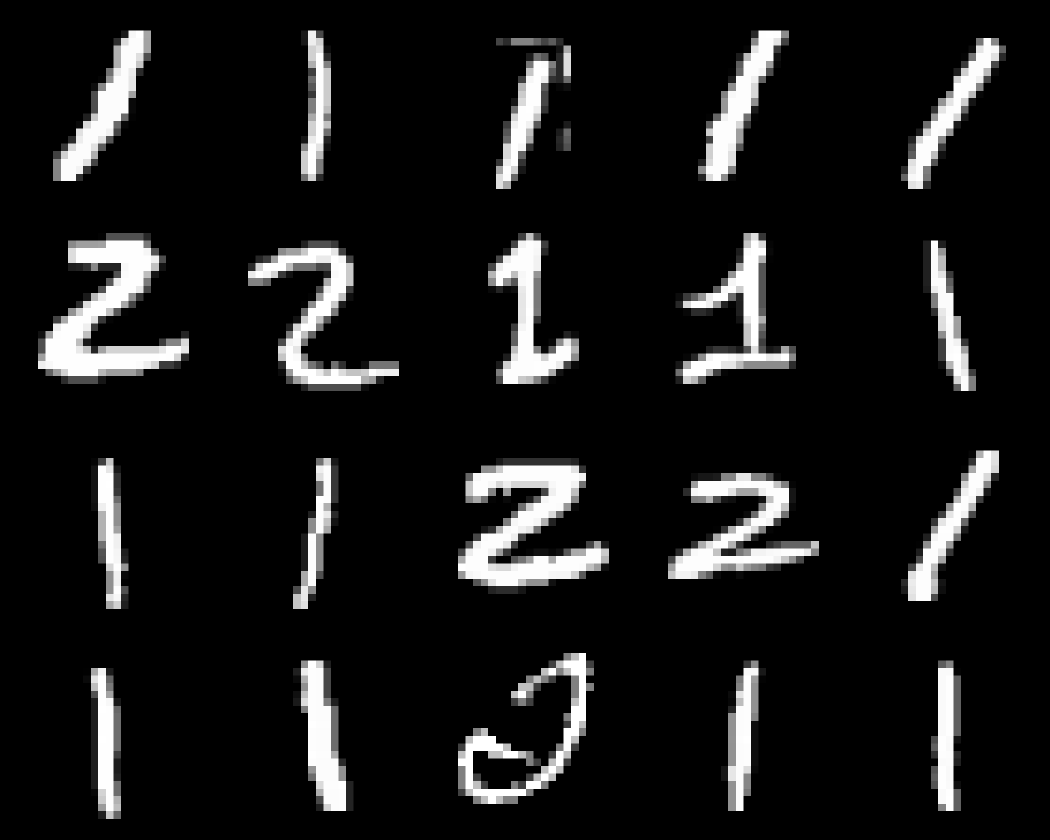}
         \vspace{5pt}
         \caption{\footnotesize Classification of 1s}
        \label{fig:class1}
  \end{minipage}%
  \hfill
  \begin{minipage}{.45\linewidth}
       \centering
        \includegraphics[width=\linewidth]{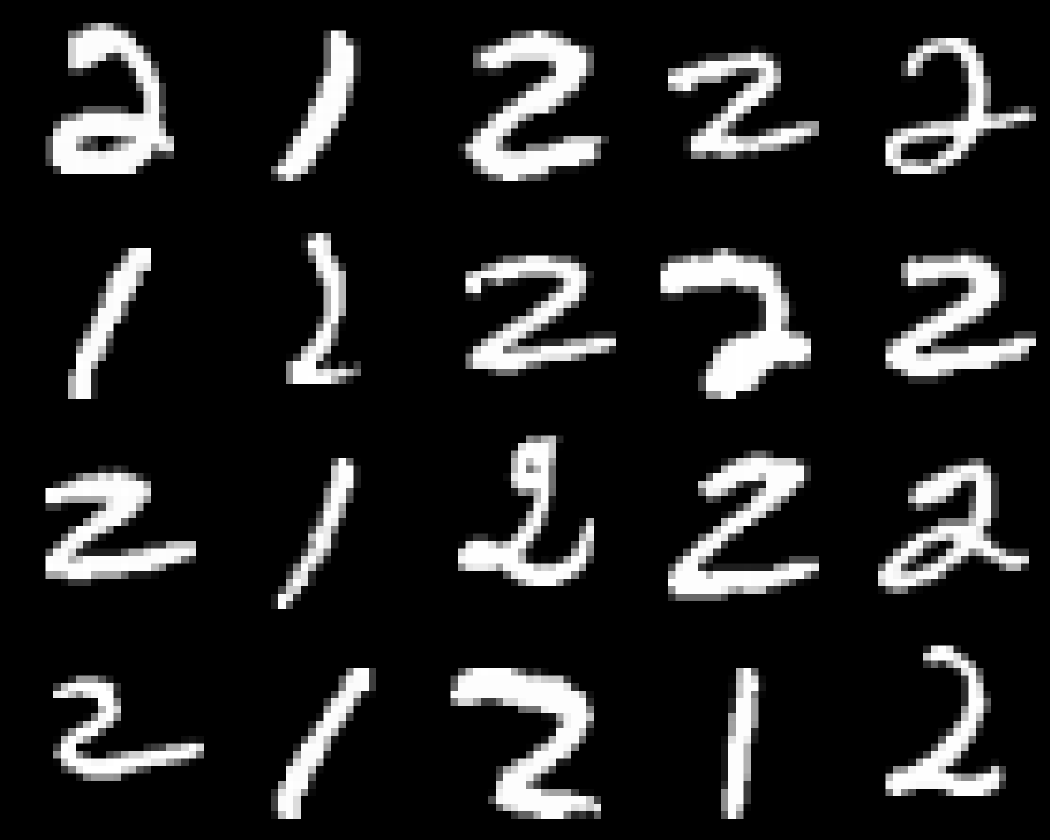}
        \vspace{5pt}
        \caption{\footnotesize Classification of 2s}
        \label{fig:class2}
  \end{minipage}
\end{figure}

In \cref{fig:class1} and \cref{fig:class2}, we show a few qualitative results on the MINIST digits dataset. We select one of the networks trained to classify the digits '1' and '2' with \nameShort~(D-Wave). Similar to the numbers reported in Tab 1 in the paper, the accuracy on those samples is $72.5\%$.

\subsection{Additional analysis of the QUBO problems using D-Wave annealer}
\begin{figure*}%
\begin{center}
\includegraphics[width=\textwidth]{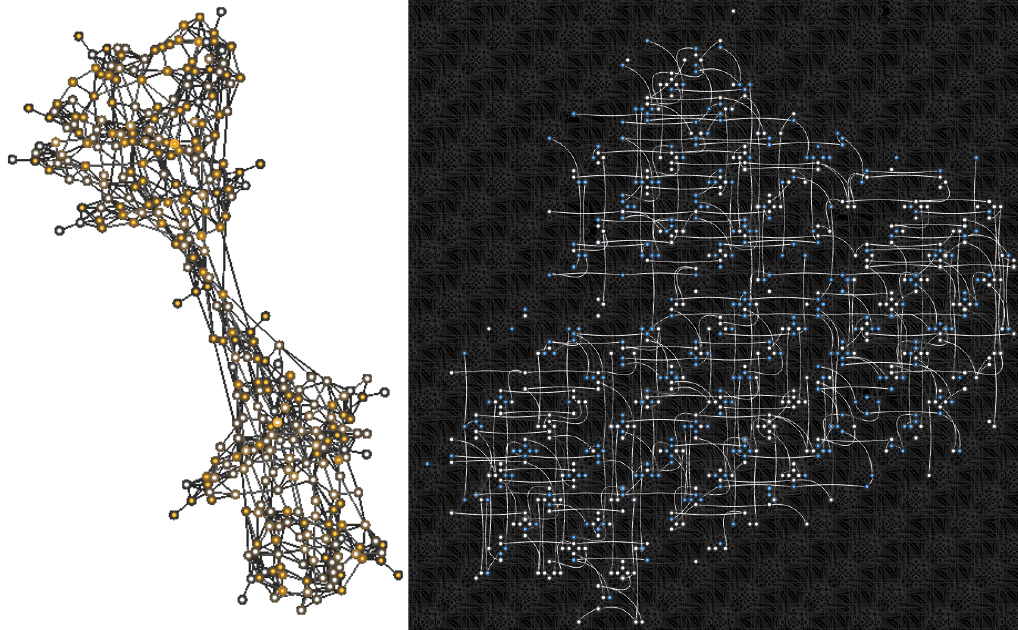}
\vspace{5pt}
    \caption{\footnotesize \textit{Left}: QUBO problem graph to update the weights of
the second layer in the first training iteration \textit{Right}: Embedded
problem graph on the Pegasus \cite{Boothby2020arXiv} D-Wave architecture.}
    \label{fig:snd_embedding}
\end{center}
\end{figure*}

\textbf{Embedding.} 
To solve QUBO problems on D-Wave, we minor-embed them onto the Pegasus architecture. As shown in~\cref{fig:snd_embedding}, this operation interprets the logical graph (left) and implements each \emph{logical bit} with multiple, redundant qubits (right). 
The graph in the figure is an exemplary QUBO corresponding to training the third layer in our MNIST experiment.

\vspace{1mm}\noindent\textbf{Evolution of the QUBO}
Next, we investigate the development of the QUBO during training. In~\cref{fig:evolution_supp} we observe how the QUBO changes over time/iterations.
The QUBO matrices shown in the figure are examples from the MNIST experiment (Sec. 3.3.2) taken from the first layer.
We display one QUBO after every $10^{\mathrm{th}}$ update.
We observe here that all QUBOs have a rather densely connected support of the otherwise sparse graph.
This overall sparsity of the graph enables us to embed those with minor embeddings and deploy them to the D-Wave QPU.

\vspace{2mm}\noindent\textbf{Sampling solutions.} 
After embedding each QUBO problem in the D-Wave Advantage architecture, we query the quantum annealer multiple times and retrieve the solution corresponding to the lowest energy. However, it is still interesting to see the solutions' distribution. To this end, we analyze QUBOs, which calculate the weight updates of the MNIST experiment in the third layer, using D-Wave Explorer and measure a series of QUBOS, which are 10 updates apart.~\cref{fig:sample_hist} plots the different energies for the sampled solutions. 
The histogram in the upper left corner corresponds to the QUBO update problem in the first iteration and the histogram in the lower right corner corresponds to the QUBO problem after 50 updates (\emph{e.g}., the $51^{\text{st}}$ update). 
In total, we collected 1000 samples from D-Wave and compared the energy.
The resulting distribution can be described as a tail-heavy Poisson distribution, which is shifted towards negative numbers. 
This type of resulting distribution is typical for problems, which are difficult to compute.

\paragraph{Similarity} In~\cref{fig:similarity_sol} we compare the similarities of the sampled D-Wave solutions.
Here the 10 scores with the lowest energy are compared with the Jaccard Similarity.
We can see that all QUBO solutions are fairly dissimilar to each other.
As seen in the picture only the first and fourth/fifth solutions have a relatively high similarity score of around $0.4$.
The indications of that are two-fold: (i) it demonstrates the importance of finding the optimal solution, and (ii) even though the D-Wave annealer cannot act like a true Bayesian posterior sampler, the samples from its (modified) posterior are diverse and span the solution space. This also confirms the findings of~\cite{QuantumSync2021} in which these samples are used as alternative solutions.

\begin{figure*}
    \centering    \includegraphics[width=\textwidth]{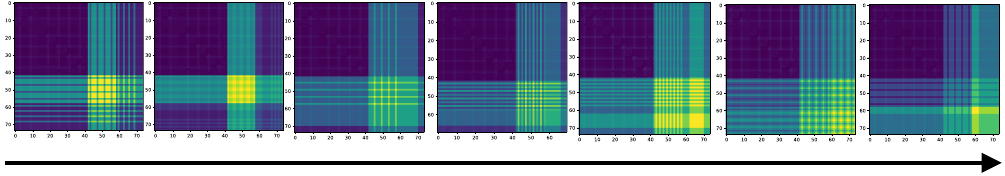}
    \vspace{5pt}
    \caption{\footnotesize Evolution of the QCBO formulation to calculate the weight updates for the first layer with the \name~algorithm.\vspace{-2mm}}
    \label{fig:evolution_supp}
\end{figure*}

\begin{figure}
   \begin{minipage}{0.65\linewidth}
       \centering
       \includegraphics[width=\textwidth]{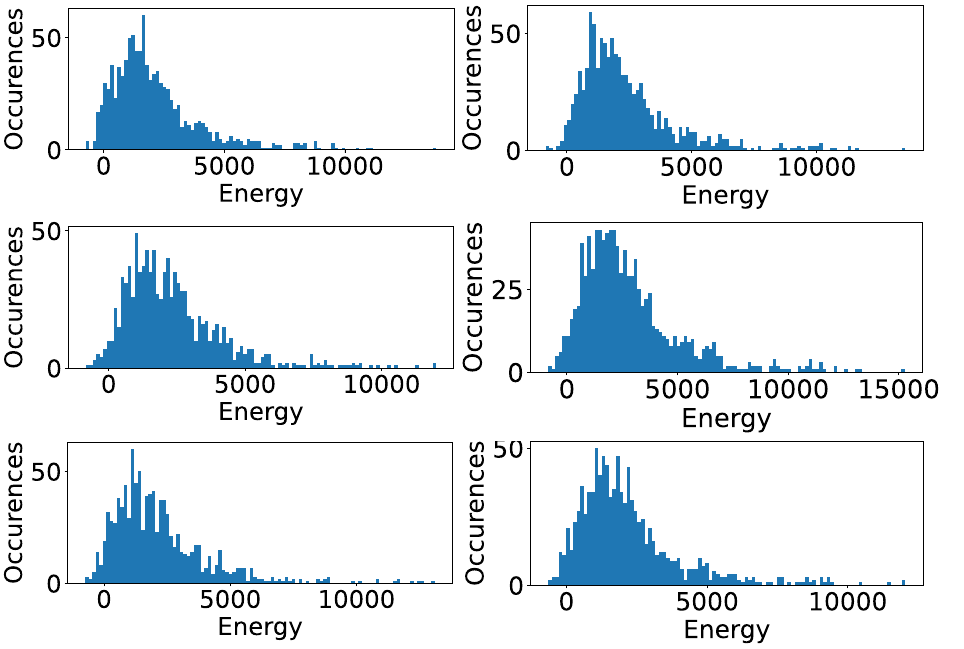}
       \vspace{5pt}
       \caption{\footnotesize Histogram of the different sampling energies of the weight update QUBO of the third layer. In the first column, we show the histogram of the zeroth, tenth, and twentieth updates, and in the second column the histogram of the thirtieth, fortieth, and fiftieth updates from top to bottom} \label{fig:similarity_sol}
   \end{minipage}%
 \hfill
   \begin{minipage}{0.25\linewidth}
   \centering
       \includegraphics[width=\textwidth]{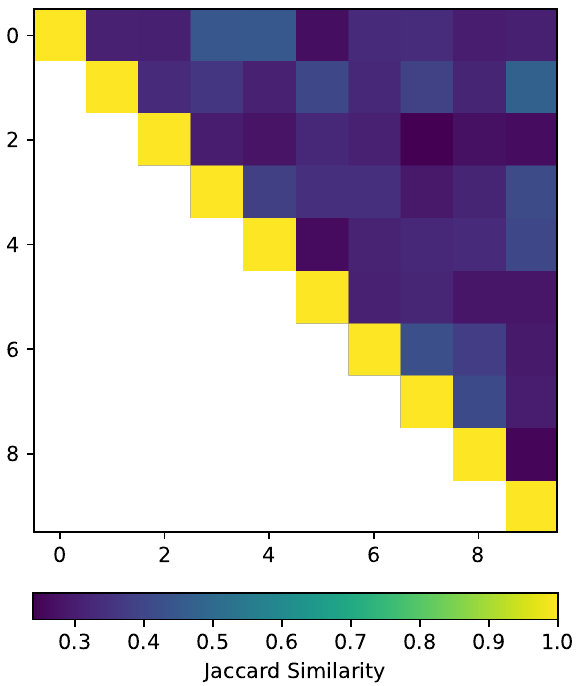}
       \vspace{5pt}
       \caption{\footnotesize Jaccard similarity of the  10 samples with the lowest energy.} 
     \label{fig:dwavedebug}
   \end{minipage}
 \label{fig:sample_hist}
\end{figure}

\paragraph{Ablation studies} 
~\cref{fig:ablation} tests the effect of different hyperparameters on Cora~\cite{mccallum2000automating}. Our hybrid optimiser, \nameShort, is robust to batch size changes, reported as training data percentages, and slightly more impacted by the step size.

  \begin{figure}%
    \centering
        \includegraphics[width = \linewidth]{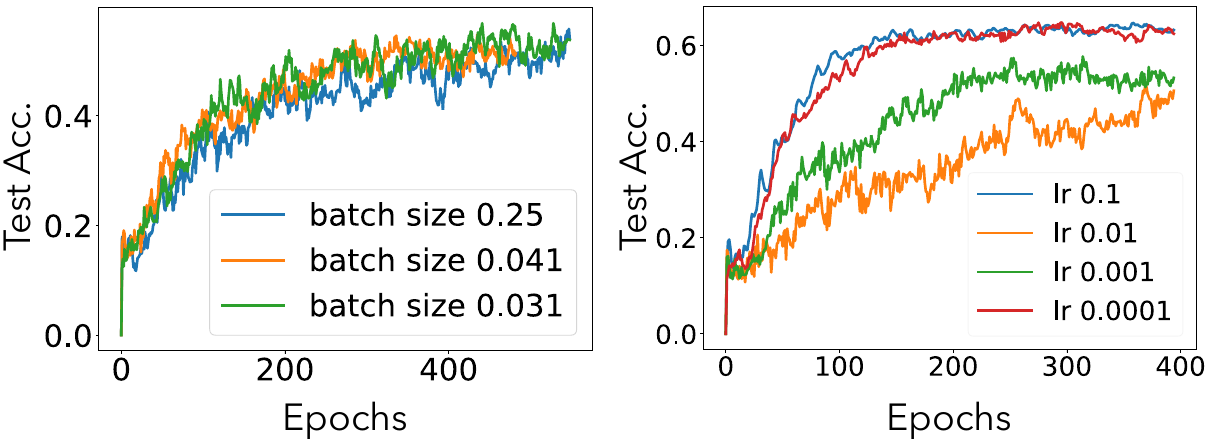}
        \vspace{5pt}
    \caption{\footnotesize Impact of learning rate (LR) and batch sizes on the test accuracy. We plot the mean accuracy over six runs.\vspace{-4mm}}
    \label{fig:ablation}
  \end{figure}

  \paragraph{Spectral gap}
A crucial feature is the \emph{spectral gap} in QUBOs solved by a quantum annealer. A larger spectral gap significantly enhances the likelihood of the annealer converging to the accurate ground truth solution \cite{mcleod2022benchmarking}.
We now examine, via simulation, the spectral gap corresponding to the QUBO in~\cref{prop:ising_for_bp_paper} for small MLPs.~\cref{fig:eigs_all} plots the eigenvalues (blue lines) of the time-varying Hamiltonian as a function of $t/T$, for a single binary linear layer $l$ with up to four neurons and batch size $1$ (besides these parameters, the size of $H$, and hence the number of distinct eigenvalues, depend on other factors; more on the feasibility of simulation below). The smallest spectral gap is highlighted by a red bar in the plots. With the same annealing time, a larger spectral gap yields a final state closer to the ground state and, thus, a higher-quality solution. 
The results show that, as expected, the spectral gap shrinks gradually with increasing problem size. While two data points are not enough to extrapolate a trend, the exercise shows how the performance of quantum annealing depends on the hyperparameters. In particular, the observations suggest keeping the number of neurons in each layer %
small.%

\begin{figure}
    \subfigure[\footnotesize $N_\ell = 1$ neuron]{ \includegraphics[width=.5\linewidth]{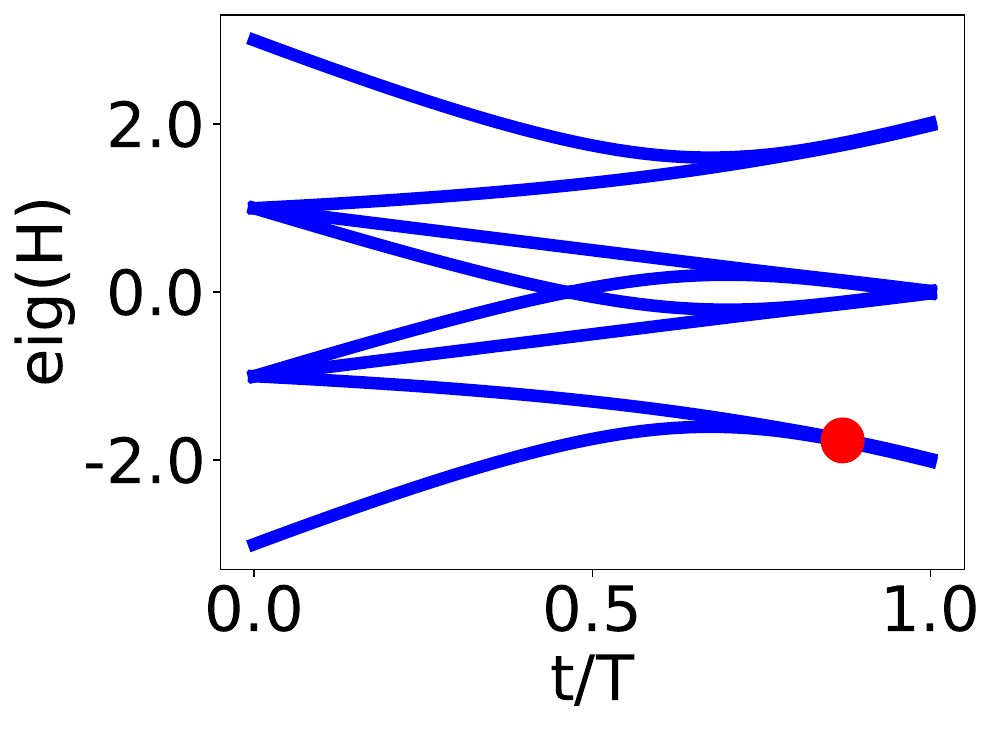}}
    \subfigure[\footnotesize $N_\ell = 2$ neurons]{\includegraphics[width=.5\linewidth]{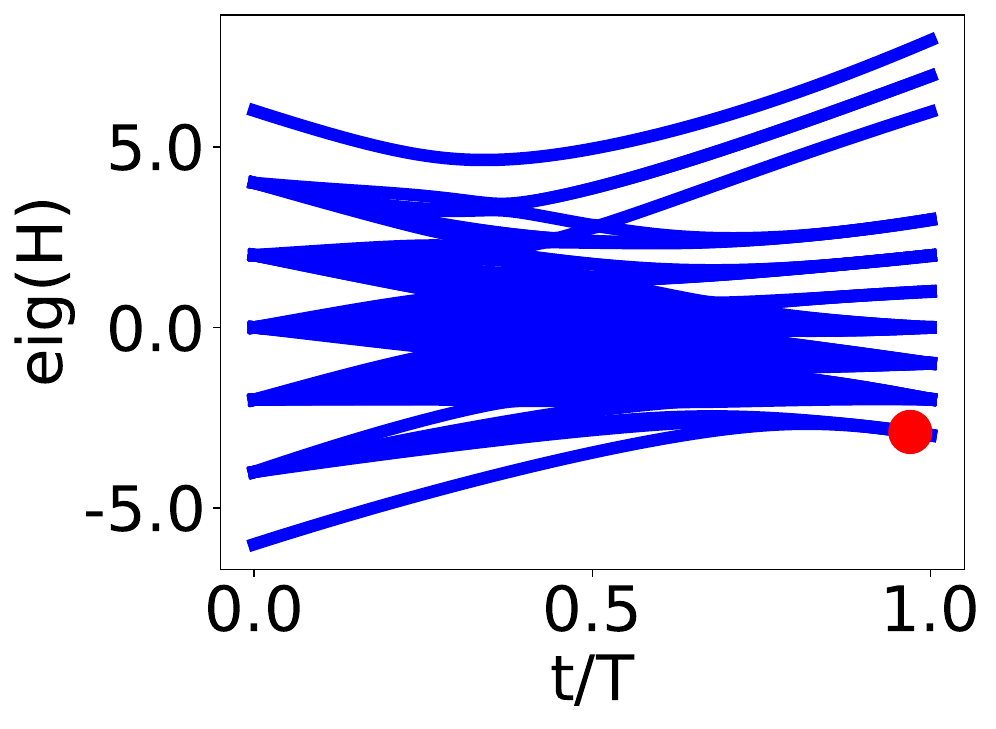}}
\vspace{5pt}
\caption{\footnotesize The eigenvalues of the QUBO Hamiltonian in~\cref{prop:ising_for_bp_paper} plotted as a function of annealing time ($t/T$) while training a binary linear layer with two neurons on the UCI-adult. The red bar represents the eigenvalue gap between the ground level and the first excited level not evolving into the ground state. The minimal spectral gap amounts to $0.03$ and $0.01$ in (a) and (b), respectively. 
    } %
\label{fig:eigs_all}
\end{figure}

%% file: content_appendix/related_work_extended.tex
\section{Extended Related Work}
\label{sec:related}

\paragraph{Binary neural networks (BNNs)}
Compared to full-precision neural networks, BNNs~\cite{BNN_2016} consume less memory and provide faster inference at the cost of certain information loss. Much effort has been devoted to mitigating performance degradation due to binarisation, as detailed in recent surveys~\cite{survey_2020,comprehensive_2021}. Here, we summarise important developments based on the recent surveys. We refer the reader to said surveys for a more complete overview.

Based on their binarisation strategies, BNN models 
can be categorised into \emph{naive} or \emph{optimisation-based} ones. BinaryConnect~\cite{binaryconnect_2015} is an early naive {training strategy}
which binarises the weights by a sign function. In this way, lightweight bitwise XNOR operations replace most real-valued arithmetic operations. BinaryNetg~\cite{BNN_2016} is an extension that binarises both weights and activations during inference and training. Dorefa-Net~\cite{dorefa-net_2018} defines new quantisation functions for weights, activations and gradients to accelerate network training. 
 
Optimisation-based BNNs attempt to address the accuracy drop more directly resulting from weight and activation binarisation in the naive models while preserving the compact nature. Following Qin et al.~\cite{survey_2020}, we further categorise such works based on their optimisation goals: 
\begin{itemize}[leftmargin=1em,itemsep=2pt,parsep=0pt,topsep=0pt]
\item \emph{Minimising the quantisation error.} XNOR-Net introduces a scaling factor for the binary weights and activations to generate a better approximation to the corresponding real-valued parameters \cite{xnor-net_2016}. XNOR-Net++ builds on this work and fuses the scaling factors of the weights and activations into one. The fused scaling factor is gained from discriminative training via backpropagation \cite{bulat2019xnor}. Dorefa-Net defines new quantisation functions for weights, activations and gradients instead of a pre-defined uniform quantiser to accelerate network training \cite{dorefa-net_2018}. LQ-Nets trains the quantisers for activations in each layer and weights in each channel jointly with the neural network \cite{zhang2018lq}. Parameterised Clipping Activation (PACT) shifts the activation distribution via learning a quantisation scaling factor during network training to achieve low bit-width precision \cite{choi2018pact}. LAB2 directly minimises binarisation loss for weights using a proximal Newton algorithm \cite{hou2016loss}.
\item \emph{Improving the loss function.} How-to-Train adds a regularisation term to encourage weight bipolarisation in addition to the task-specific loss term \cite{tang2017train}. LAB directly develops a loss function during layer-wise weight binarisation \cite{hou2016loss}. ReActNet introduces a spatial distributional loss to enforce the BNN output distribution to approximate that of the full-precision model, on top of activation distribution shift and reshaping using new quantisation functions \cite{liu2020reactnet}.
\item \emph{Reducing the gradient approximation error.} Another stream of work focuses on the nature of quantisation and activation functions to enable backpropagation and enhance forward propagation. Bi-Real Net approximates the derivative of the sign function for activations using a piecewise linear function to overcome the intrinsic discontinuity and proposes a magnitude-aware gradient for weights \cite{liu2018bi}. HWGQ-Net designs a forward quantised approximation and a backward continuous approximation of the ReLU activation function to address the learning inefficiency of BNNs \cite{cai2017deep}. 
\end{itemize}
Various training techniques including asymptotic optimisation, quantisation, gradient approximation and network structure transformation, are adopted in some of the learning-based models above. Standard techniques are used in optimisation, involving Adam, AdaMax, SGD and RMSprop. For more details please refer to Qin \etal~\cite{survey_2020}.

\paragraph{Quantum annealing}
Recent progress in \emph{quantum computer vision}~\cite{QuantumSync2021} has shown that various computer vision algorithms could benefit from QUBO-formulations and quantum annealing~\cite{QuantumSync2021,Chin_2020_ACCV,SeelbachBenkner2020, zaech2022adiabatic,Arrigoni2022,mccormick2022multiple,yurtsever2022q,Farina2023, Bhatia2023CCuantuMM}. Arthur~\etal~\cite{arthur2021qubo} has similarly shown the advantage for several machine learning models. A new paradigm was developed recently to harness the computational power of quantum annealers for the training of BNNs~\cite{sasdelli2021quantum} as well as for optimising in binary variables~\cite{yurtsever2022q}. As an early work in this direction, a simple version of BNN was examined. As alluded to above, our work improves the scalability of their approach by adopting a hybrid and incremental training scheme.

\paragraph{Quantum neural networks} %
As an emerging field, \emph{quantum machine learning}~\cite{biamonte2017quantum,schuld2015introduction,broughton2020tensorflow,wittek2014quantum} has shown that the training of linear regression, support vector machines and k-means clustering admit QUBO-like formulations~\cite{arthur2021qubo}, whereas the first neural network variants trained on quantum hardware were Boltzmann Machines. In parallel, \emph{quantum deep learning}~\cite{wiebe2014quantum,garg2020advances,kerenidis2019quantum}, the problem of creating quantum circuits that enhance the operations of neural networks by physically realising them, has emerged to alleviate some of the computational limitations of classical deep learning, thanks to the efficient training algorithms~\cite{kerstin2019efficient,beer2020training}. However, a severe challenge remains to be implementing non-linearities and other non-unitary operations with quantum unitaries~\cite{schuld2014quest,kerenidis2019quantum}. Similarly, \emph{quantum convolutional neural networks} (Q-CNNs)~\cite{cong2019quantum} have been developed with the mindset of implementing the analogue of classical CNN operations via quantum gates. These works are fundamentally different to ours, as they try to realise quantum implementations of machine learning algorithms, whereas we leverage quantum computation to solve the classical problem of neural network training.

\paragraph{Network quantisation}
Instead of training BNNs from data, network quantisation techniques aim to reduce the precision (including to binary) of a trained full-precision neural network in a manner that preserves accuracy; see Gholami et al. \cite{gholami2021survey} for a recent survey. We stress that \emph{quantisation} fundamentally differs in formulating a BNN for execution or training on \emph{quantum} computers.
In the current work, we are focusing on the most extreme kind of network quantisation: binary neural networks. More general types of quantisation can be formulated as a QUBO in principle~\cite{nagel2021white}.

\paragraph{Binary graph neural networks}
As opposed to the plethora of methods dealing with the binarisation and quantisation of ordinary neural networks, binary graph neural networks are under-explored. Only a few works exist~\cite{wang2021bi,bahri2021binary,wang2021binarized} and most of them operate on the principles of XNOR-Net and its variants, aiming to binarise the multiplication. Still, re-quantisation remains to be a necessary operation to avoid full-precision computations~\cite{yao2022fullprecision}. On a parallel track, developing efficient binary graph operators~\cite{jing2021meta} or quantising graph neural networks~\cite{zhu2023rm,jing2021meta,huang2022epquant} as well as GNN-compression are somewhat studied. We refer the reader to a recent survey~\cite{wei2022graph} and find it worth mentioning that none of these procedures are friendly for implementation on quantum hardware.